%% file: boosting_journal_6.tex
\documentclass[twoside,11pt]{article}

\usepackage{jmlr2e}

\usepackage{amsmath,amssymb,amsfonts,mathtools}
\usepackage{algorithm}
\usepackage{algorithmic}
\usepackage{enumitem}
\usepackage{multirow}
\usepackage{array}
\usepackage{bm}
\usepackage{color}
\usepackage{epstopdf}
\usepackage{pgfplots}
\input{macro}

\definecolor{amethyst}{rgb}{1, 0, 1}
\definecolor{blue-violet}{rgb}{0.54, 0.17, 0.89}
\definecolor{brightturquoise}{rgb}{0.03, 0.91, 0.87}

\jmlrheading{1}{20XX}{1-19}{12/16}{XX/XX}{Aleksandr Y. Aravkin, Giulio Bottegal, and Gianluigi Pillonetto}

\ShortHeadings{Boosting as a Kernel-Based Method}{Aravkin, Bottegal and Pillonetto}
\firstpageno{1}

\begin{document}

\title{Boosting as a Kernel-Based Method}

\author{\name Aleksandr Y.\ Aravkin \email saravkin@uw.edu \\
       \addr Department of Applied Mathematics \\
			University of Washington \\
			Seattle, WA 98195-4322, USA 		
       \AND
       \name Giulio\ Bottegal \email giulio.bottegal@gmail.com \\
       \addr Department of Electrical Engineering\\
       TU Eindhoven \\
       Eindhoven, MB 5600, The Netherlands
       \AND
       \name Gianluigi Pillonetto \email giapi@dei.unipd.it \\
       \addr Department of Information Engineering \\
       University of Padova \\
       Padova, 35131, Italy}

\editor{XXX}

\maketitle

\begin{abstract}
Boosting combines weak (biased) learners to obtain effective learning algorithms
for classification and prediction. 
In this paper, we show a
connection between boosting and kernel-based methods, 
highlighting both theoretical and practical applications. 
In the context of $\ell_2$ boosting, we start with a weak 
linear learner defined by a kernel $K$. We show that boosting with this learner 
is equivalent to estimation with a special {\it boosting kernel} that depends on $K$,
as well as on the regression matrix, noise variance, and hyperparameters.
The number of boosting iterations is modeled as a continuous hyperparameter, 
and fit along with other parameters using standard techniques.\\
We then generalize the boosting kernel to a broad new class of boosting approaches
for more general weak learners, 
including those based on the $\ell_1$, hinge and Vapnik losses.
The approach allows fast hyperparameter tuning for this general class, 
and has a wide range of applications, including robust regression and classification. 
We illustrate some of these applications with numerical examples on synthetic and real data. 
\end{abstract}


\section{Introduction}

Boosting is a popular technique to construct learning algorithms \citep{schapire2003boosting}. The basic idea is that any \emph{weak learner}, i.e. algorithm that 
is only slightly better than guessing, can be used to build an effective learning mechanism that achieves high accuracy. 
Since the introduction of boosting in Schapire's seminal work \citep{schapire1990strength},
numerous variants have been proposed for regression, classification, and specific applications including semantic learning and computer vision  \citep{schapire2012boosting,viola2001fast,temlyakov2000weak,tokarczyk2015features,bissacco2007fast,cao2014face}.
In particular, in the context of classification, \emph{LPBoost}, \emph{LogitBoost} \citep{friedman2000additive}, \emph{Bagging and Boosting} \citep{lemmens2006bagging} and \emph{AdaBoost} \citep{freund1997decision} have become standard tools, the latter having being recognized as the best off-the-shelf binary classification method \citep{breiman1998arcing,zhu2009multi}. Applications of the boosting principle are also found in decision tree learning \citep{tu2005probabilistic} and distributed learning \citep{fan1999application}. For a survey on applications of boosting in classification tasks see the work of~\cite{freund1999short}. 
For regression problems, \emph{AdaBoost.RT} \citep{solomatine2004adaboost,avnimelech1999boosting} and \emph{$\ell_2$ Boost} \citep{buhlmann2003boosting,tutz2007boosting,champion2014sparse} are the most prominent boosting algorithms. 
In particular, in $\ell_2$ boosting the weak learner often corresponds to a kernel-based estimator with a heavily weighted regularization term. The fit on the training set is then measured using the quadratic loss and increases at each iteration.
Hence, the procedure can lead to overfitting if it continues too long \citep{buhlmann2007boosting}. To avoid this, several stopping criteria based on model complexity arguments have been developed. \cite{hurvich1998smoothing} propose a modified version of Akaike's information criterion (AIC); \cite{hansen2001model} use the principle of minimum description length (MDL), and \cite{buhlmann2003boosting} suggest a five-fold cross validation.\\
In this paper, we focus on $\ell_2$ boosting and
consider linear weak learners induced by the combination of a quadratic loss 
and a regularizer induced by a kernel $K$. We show that the resulting boosting estimator  
is equivalent to estimation with a special {\it boosting kernel} that depends on $K$,
as well as on the regression matrix, noise variance, and hyperparameters. 
This viewpoint leads to both greater generality and better computational efficiency. 
In particular, the number of boosting iterations $\nu$ is a continuous hyperparameter of the boosting kernel, 
and can be tuned by standard fast hyper-parameter selection techniques including SURE, 
generalized cross validation, and marginal likelihood \citep{hastie2001elements}.
In Section~\ref{sec:experiments}, 
we show that tuning $\nu$ is far more efficient than applying boosting iterations,
and non-integer values of $\nu$ can improve performance.\\  
We then generalize the boosting kernel to a wider class of problems, including robust regression, 
by combining the boosting kernel with piecewise linear quadratic (PLQ) loss functions (e.g. $\ell_1$, Vapnik, Huber). 
The computational burden of standard boosting is high for general loss functions, 
since the estimator at each iteration is no longer a linear function of the data. 
The boosting kernel makes the general approach tractable. 
We also use the boosting kernel in the context of regularization problems in reproducing kernel Hilbert spaces (RKHSs), 
e.g. to solve classification formulations that use the hinge loss.\\
The organization of the paper is as follows. After a brief overview of boosting in regression and classification, 
we develop the main connection between boosting and kernel-based methods in the context of finite-dimensional inverse problems in Section~\ref{sec:boosting_kernel}. 
Consequences of this connection are presented in Section \ref{sec:consequences}. 
In Section~\ref{sec:new_class_boosting} we combine the boosting kernel with PLQ penalties
to develop a new class of boosting algorithms. We also consider regression and classification in RKHSs.
In Section~\ref{sec:experiments} we show numerical results for several experiments involving the boosting kernel. 
We end with discussion and conclusions in Section~\ref{sec:conclusions}.

\section{Boosting as a kernel-based method} \label{sec:boosting_kernel}

In this section, we give a basic overview of boosting, and present the boosting kernel. 

\subsection{Boosting: notation and overview}\label{SecBS}

Assume we are given a model $g(\theta)$ for some observed data $y\in\mathbb{R}^n$,
where $\theta\in\mathbb{R}^m$ is an unknown parameter vector. 
Suppose our estimator $\hat \theta$ for $\theta$ minimizes some objective that balances variance with bias. 
In the boosting context, the objective is designed to provide a {\it weak estimator}, i.e.  
one with low variance in comparison to the bias. 

Given a loss function $\mathcal{V}$ and a kernel matrix $K\in\mathbb{R}^{m\times m}$, the weak estimator 
can be defined by minimizing the regularized formulation
\begin{equation}\label{Weak2}
\hat \theta :=  \arg\min_\theta \left\{ J(\theta; y) :=  \mathcal{V}(y-g(\theta)) + \gamma \theta^T K^{-1} \theta\right\},
\end{equation} 
where the regularization parameter $\gamma$ is large and leads to over-smoothing. 
Boosting uses this weak estimator iteratively, as detailed below. 
The predicted data for an estimator $\hat \theta$ are  denoted by $\hat{y}=g(\hat{\theta})$.

{\bf{Boosting scheme:}}
\begin{enumerate}
\item Set $\nu=1$ and obtain $\hat{\theta}(1)$ and $\hat{y}(1)=g(\hat{\theta}(1))$ using (\ref{Weak2});
\item Solve (\ref{Weak2}) using the current residuals as data vector, i.e. compute 
$$
\hat\theta(\nu) = \argmin_{\theta} \ J(\theta;y-\hat{y}(\nu)),
$$
and set the new predicted output to 
$$
\hat{y}(\nu+1) =  \hat{y}(\nu) + g(\hat\theta(\nu)). 
$$
\item Increase $\nu$ by 1 and repeat step 2 for a prescribed number of iterations.
\end{enumerate}

\subsection{Using regularized least squares as weak learner} 

Suppose data $y$ are generated according to 
\begin{equation}\label{MeasMod}
y = U\theta + v, \quad v \sim \mathcal N (0, \sigma^2 I), 
\end{equation}
where $U$ is a known regression matrix of full column rank. 
The components of $v$ are independent random variables, mean zero and variance $\sigma^2$.

We now use a quadratic loss to define the regularized weak learner. 
Let  $\lambda$ to denote the kernel scale factor and set
$\gamma=\sigma^2 / \lambda$ so that 
(\ref{Weak2}) becomes 
\begin{equation}\label{Eq1}
\hat \theta = \arg\min_{\theta} \|y - U\theta\|^2 + \frac{\sigma^2}{\lambda}\theta^T  K^{-1}\theta.
\end{equation}
We obtain the following expressions for the predicted data $\hat y = U \hat \theta$:
\begin{eqnarray}\nonumber 
\hat{y} &=& \arg\min_{f} \left\{\|y - f\|^2 + \sigma^2 f^T P_{\lambda}^{-1} f \right\}\\  \label{Eq2}
&=&   P_{\lambda} (P_{\lambda} + \sigma^2 I)^{-1}y,
\end{eqnarray}
where 
\begin{equation}\label{Eq3}
P_{\lambda} = \lambda U K U^T 
\end{equation}
is assumed invertible for the moment. This assumption will be relaxed later on.

The following well-known connection \citep{Wahba1990} between (\ref{Eq1}) and Bayesian estimation
is useful for theoretical development.
Assume that $\theta$ and $v$ are independent Gaussian random vectors with priors 
\[
\theta \sim \mathcal{N}(0, \lambda K), \quad v \sim \mathcal{N}(0, \sigma^2 I).
\]
Then, (\ref{Eq1}) and  (\ref{Eq2}) provide the minimum variance estimates of $\theta$ and $U \theta$ conditional 
on the data $y$. In view of this, we refer to diagonal values of $K$ as the {\it prior variances}
of $\theta$.

\subsection{The boosting kernel}

Define 
\begin{equation}
\label{eq:Slam}
S_{\lambda} =  P_{\lambda} (P_\lambda + \sigma^2 I)^{-1}.
\end{equation}
Fixing a small $\lambda$, the predicted data obtained by the weak kernel-based learner is
\[
\hat y(\nu = 1) = S_{\lambda}y,
\]
where $\nu$ is the number of boosting iterations.
According to the scheme specified in Section \ref{SecBS}, 
as $\nu$ increases, boosting refines the estimate as follows:
\begin{eqnarray}\nonumber
\hat y(2) &=& S_{\lambda}y+S_{\lambda}(I-S_{\lambda})y \\ \nonumber
\hat y(3) &=& S_{\lambda}y+S_{\lambda}(I-S_{\lambda})y + S_{\lambda}(I-S_{\lambda})^2y \\ \nonumber 
&\vdots& \\  \label{BoostEst}
 \hat y(\nu) &=& S_{\lambda}\sum_{i = 0}^{\nu-1} \left( I - S_{\lambda}\right)^{i} y. 
\end{eqnarray} 
 
We now show that the boosting estimates $\hat y(\nu)$  are kernel-based estimators 
from the {\it boosting kernel}, which plays a key role for subsequent developments.
 
\begin{proposition}\label{BoostKer}
The quantity $\hat y(\nu)$ is a kernel-based estimator
\begin{equation*} 
\hat{y}(\nu)  =  S_{\lambda, \nu}y = P_{\lambda,\nu }(P_{\lambda,\nu} + \sigma^2 I)^{-1} y,
\end{equation*}
where $P_{\lambda,\nu }$ is the boosting kernel defined by
\begin{eqnarray}
P_{\lambda, \nu} \nonumber 
&=& \sigma^2 \left( I - P_{\lambda}\left(P_{\lambda} + \sigma^2 I\right)^{-1}\right)^{-\nu}-\sigma^2I \\ \label{BoostingKer}
& =& \sigma^2 \left( I - S_{\lambda}\right)^{-\nu}-\sigma^2I.
\end{eqnarray}
\end{proposition}
\begin{proof}
First note that  $S_\lambda$ satisfies 
\begin{equation}
\label{eq:SlamId}
S_\lambda = P_{\lambda}\left(P_{\lambda} + \sigma^2 I\right)^{-1} 
= 
I - \sigma^2\left(P_{\lambda} + \sigma^2 I \right)^{-1}.
\end{equation}
This follows simply from adding the term $\sigma^2\left(P_{\lambda} + \sigma^2\right)^{-1}$ to~\eqref{eq:Slam}
and observing that expression reduces to $I$. Next, plugging in the expression~\eqref{BoostingKer} for $P_{\lambda, \nu}$  into 
the right hand side of expression~\eqref{eq:SlamId} for $S_{\lambda, \nu}$, 
we have
\[
\begin{aligned}
S_{\lambda, \nu} &= I - \sigma^2\left(P_{\lambda,\nu} + \sigma^2 I \right)^{-1}\\
& = I - \sigma^2\left(\sigma^2 \left( I - S_{\lambda}\right)^{-\nu} \right)^{-1}\\
& = I - \left( I - S_{\lambda}\right)^{\nu} \\
& = S_{\lambda}\sum_{i = 0}^{\nu-1} \left( I - S_{\lambda}\right)^{i},
\end{aligned}
\]
exactly as required by 
(\ref{BoostEst}).
\end{proof}
In Bayesian terms, for a given $\nu$, the above result also shows that
boosting returns the minimum variance estimate of the noiseless output $f$ conditional on $y$ if $f$ and $v$ are independent Gaussian random vectors with priors 
\begin{equation} \label{Eq5}
f \sim \mathcal{N}(0, P_{\lambda,\nu}), \quad v \sim \mathcal{N}(0, \sigma^2 I). 
\end{equation}

\section{Consequences} \label{sec:consequences}

In this section, we use Proposition \ref{BoostKer} to gain new insights on boosting
and a new perspective on hyperparameter tuning.

\subsection{Insights on the nature of boosting}\label{Sec3.1}

We first derive a new representation of the boosting kernel $P_{\lambda, \nu}$ via 
a change of coordinates. Let $VDV^T$ be the SVD of $UKU^T$.  
Then, we obtain 
\begin{eqnarray}\nonumber
P_{\lambda, v}  &=& \frac{\sigma^2}{(\sigma^2)^\nu} \left(\lambda U K U^T + \sigma^2 I\right)^{\nu} - \sigma^2 I \\ \label{BoostingKer2}
& =&  \sigma^2 V \left[ \left(\frac{\lambda D + \sigma^2 I}{\sigma^2}\right)^{\nu}-I\right]V^T
\end{eqnarray}
and the predicted output can be rewritten as 
\[
\hat{y}(\nu) = V\left(I- \sigma^{2\nu} \left(\lambda D + \sigma^2 I\right)^{-\nu}\right)V^Ty.
\]
In coordinates $z = V^Ty$, the estimate of each component of $z$ is 
\begin{equation}\label{BoostEstzi}
\hat{z}_i(\nu) = \left(1 - \frac{\sigma^{2\nu}}{\left(\lambda d_i^2 + \sigma^2\right)^\nu}\right)z_i,
\end{equation}
and corresponds to the regularized least squares estimate induced by a diagonal kernel with $(i,i)$ entry 
\begin{equation}\label{BoostPriorVar}
\sigma^2 \left(\frac{\lambda d_i^2}{\sigma^2}+1\right)^\nu - \sigma^2.
\end{equation}
In Bayesian terms, (\ref{BoostPriorVar}) is the prior variance assigned by boosting 
to the noiseless output $V^T U \theta$.

Eq. \eqref{BoostPriorVar}  shows that boosting
builds a kernel on the basis of the output signal-to-noise ratios
$SNR_i = \frac{\lambda d_i^2}{\sigma^2}$, which then enter  
$\left(\frac{\lambda d_i^2}{\sigma^2}+1\right)^\nu$. 
All diagonal kernel elements
with $d_i>0$ grow to $\infty$ as $\nu$
increases; therefore asymptotically, data will be perfectly interpolated
but with growth rates controlled by the $SNR_i$. 
If $SNR_i$ is large, the prior
variance increases quickly and  
after a few iterations the estimator is essentially unbiased along the
$i$-th direction. If $SNR_i$ is close to zero, the $i$-th
direction is treated as though affected by ill-conditioning, and a large $\nu$
is needed to remove the regularization on $\hat{z}_i(\nu)$.
 
This perspective makes it clear when boosting can be effective. 
In the context of inverse problems (deconvolution),
$\theta$ in (\ref{MeasMod}) represents the unknown input to a linear system whose impulse response defines the 
regression matrix $U$. For simplicity, assume that the kernel $K$ is set to the identity matrix,
so that the weak learner (\ref{Eq1}) becomes ridge regression  and the $d_i^2$ in (\ref{BoostPriorVar}) reflect 
the power content of the impulse response at different frequencies. 
Then, \emph{boosting can outperfom standard ridge regression 
if  the system impulse response and input share a similar power spectrum}.
Under this condition, 
boosting can inflate the prior variances
(\ref{BoostPriorVar}) along  
the right directions. 
For instance, if the impulse response energy is located at low frequencies,
as $\nu$ increases  boosting will amplify the low pass
nature of the regularizer. This can significantly improve the estimate
if the input is also low pass.

\subsection{Hyperparameter estimation}\label{Sec3.2}

In the classical scheme described in section \ref{SecBS},
$\nu$ is an iteration counter that only takes integer values,
and the boosting scheme is sequential: to obtain the estimate $\hat{y}(\nu=m)$,
one has to solve $m$ optimization problems.
Using (\ref{BoostingKer}) and (\ref{BoostingKer2}),
we can interpret $\nu$ as a kernel hyperparameter, and let it 
take real values.
In the following we estimate both the scale factor $\lambda$
and $\nu$ from the data, and restrict the range of $\nu$
to $\nu \geq 1$.

The resulting boosting
approach estimates $(\lambda,\nu)$  by minimizing  fit measures such as cross validation or SURE~\citep{hastie2001elements}.
In particular, this accelerates the tuning procedure, as it requires solving a single problem instead of 
multiple boosting iterations. 
Consider estimating $(\lambda,\nu)$ using the SURE method. Given $\sigma^2$ 
(e.g. using an unbiased estimator), choose 
\begin{equation}\label{SURE}
(\hat \lambda,\hat \nu) = \arg\min_{\lambda \geq 0, \nu \geq 1} \ \|y - \hat y(\nu)\|^2 + 2 \sigma^2 \mbox{trace}(S_{\lambda,\nu}).
\end{equation}
Straightforward computations show that, for the cost of a single SVD, 
problem~\eqref{SURE} simplifies to
\begin{equation} \label{eq:SURE_2}
(\hat \lambda,\hat \nu) = \arg\min_{\lambda\geq 0,\nu\geq 1} 
\sum_{i=1}^n \frac{z_i^2\sigma^{4\nu}}{\left(\lambda d_i^2 + \sigma^2\right)^{2\nu}} + 
2 \sigma^2 n- \sum_{i=1}^n \frac{2 \sigma^{2\nu+2}}{(\lambda d_i^2 + \sigma^2)^\nu},
\end{equation}
which is a smooth 2-variable problem over a box, and can be easily optimized. 

We can also extract some useful information on the nature of the optimization problem (\ref{eq:SURE_2}).
In fact, denoting $J$ the objective, we have 
\begin{align} \label{eq:der_J}
\frac{\partial J}{\partial \nu} & = 2 \sum_{i=1}^n  \log (\alpha_i) z_i^2 \alpha_i^{2\nu} - 2\sigma^2\sum_{i=1}^n  \log (\alpha_i) \alpha_i^{\nu} \nonumber \\
& = 2 \sum_{i=1}^n \log (\alpha_i) \alpha_i^{\nu} ( z_i^2 \alpha_i^{\nu} - \sigma^2) \,,
\end{align}
where we have defined $\alpha_i := \frac{\sigma^2}{\lambda d_i^2 + \sigma^2} \,.$ 
Simple considerations on the sign of the derivative then show that
\begin{itemize}
\item if 
\begin{equation} \label{eq:lambda_min}
\lambda < \min_{i=1,\ldots,n}\frac{z_{ i}^2 - \sigma^2}{d_{ i}^2}, 
\end{equation}
then $\hat \nu = +\infty$. This means that we have chosen a learner so weak that SURE suggests an infinite 
number of boosting iterations as optimal solution;
\item if 
\begin{equation} \label{eq:lambda_max}
\lambda > \max_{i=1,\ldots,n} \frac{z_{ i}^2 - \sigma^2}{d_{ i}^2}, 
\end{equation}
then $\hat \nu = 1$. This means that the weak learner is instead so strong that SURE suggests not to perform any boosting iterations.
\end{itemize}

\subsection{Numerical illustration} 

\begin{figure}
  \begin{center}
   \begin{tabular}{cc}
\hspace{-.2in}
 { \includegraphics[scale=0.4]{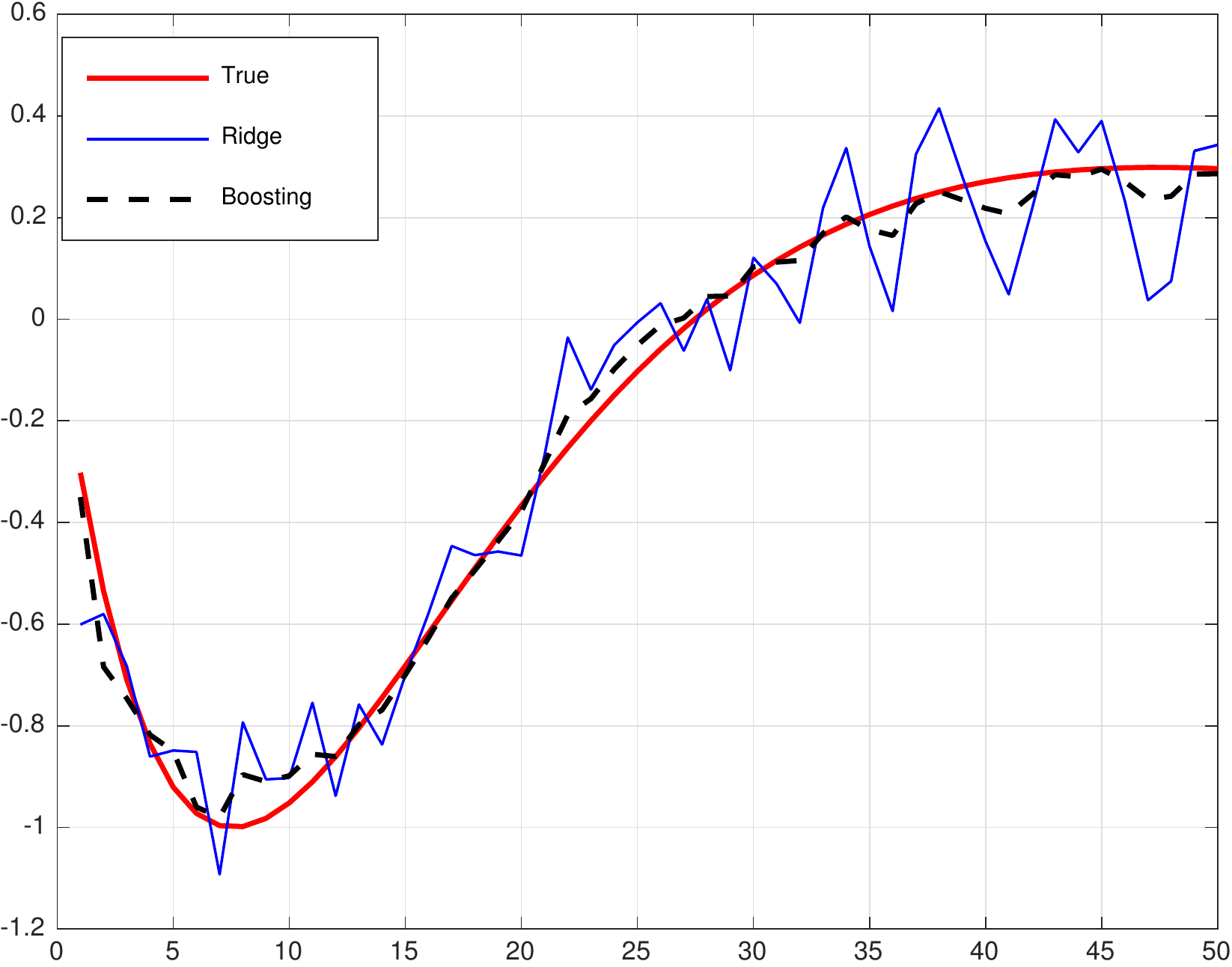}}
    \end{tabular}
    \caption{True signal (thick red line), Ridge estimate (solid blue) and Boosting estimate (dashed black)
    obtained in the first Monte Carlo run. The system impulse response is a low pass signal.}  \label{Fig1}
     \end{center}
\end{figure}

\begin{figure*}
  \begin{center}
   \begin{tabular}{cc}
\hspace{-.2in}
 { \includegraphics[scale=0.35]{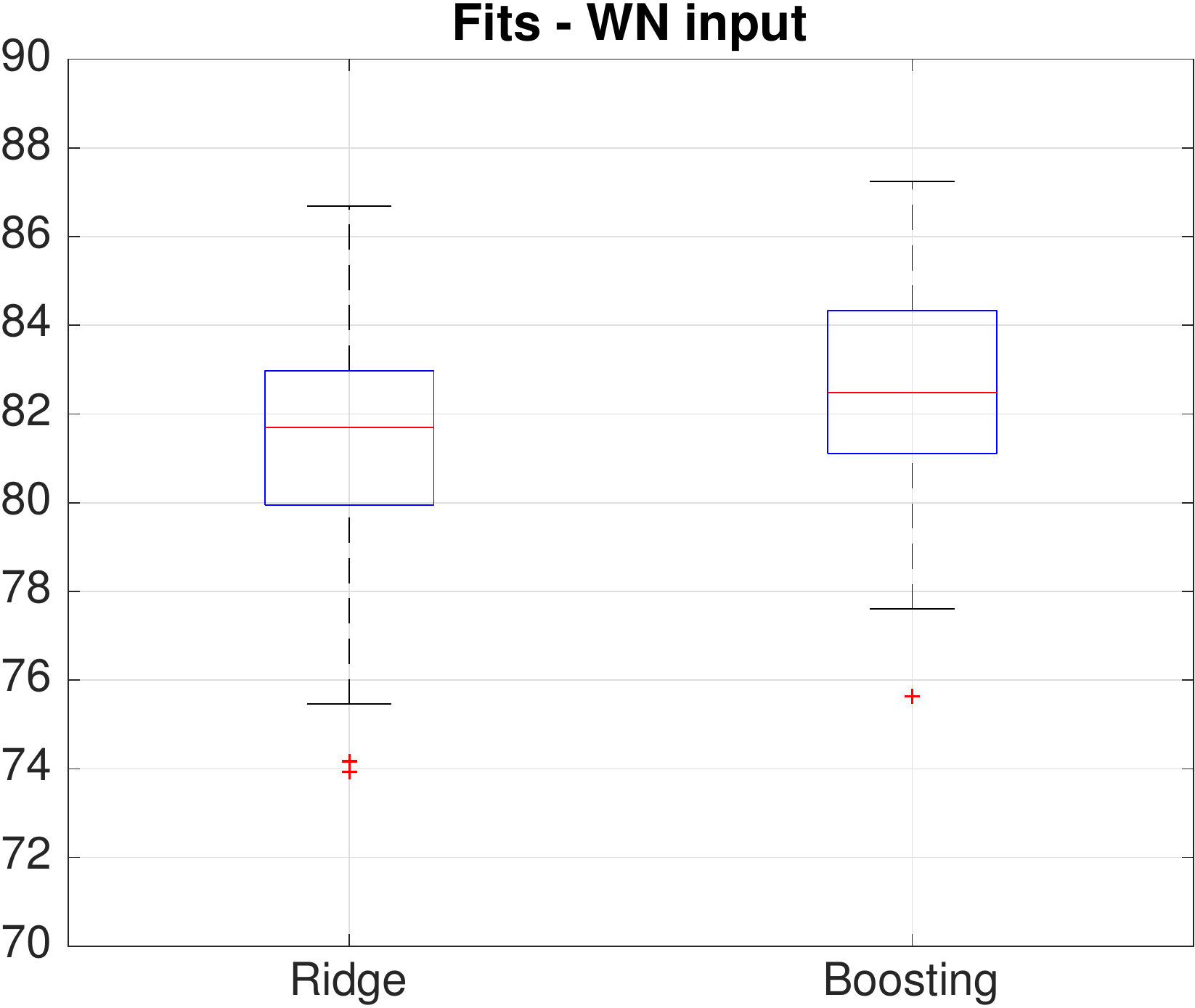}}
\hspace{.1in}
 { \includegraphics[scale=0.35]{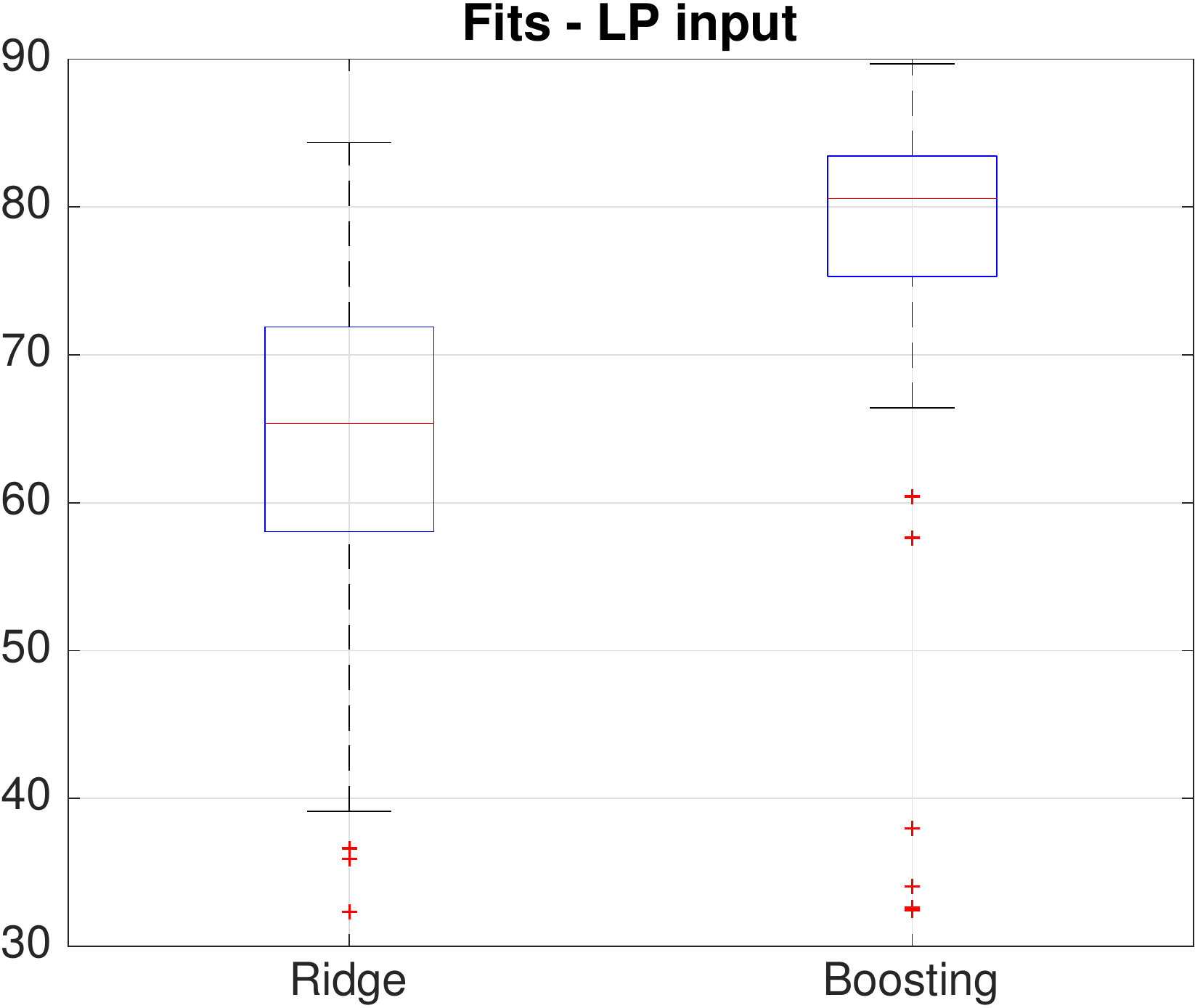}}
    \end{tabular}
    \caption{Boxplot of the percentage fits obtained by Ridge regression and Boosting, using SURE
    to estimate hyperparameters; system impulse response is white noise (left) and low pass (right).}  \label{Fig2}
     \end{center}
\end{figure*}

We illustrate our insights using a numerical experiment.
Consider (\ref{MeasMod}), where $\theta \in \mathbb{R}^{50}$ represents the input to a discrete-time linear system.
In particular, the signal is taken from \citep{Wahba1990} and displayed in Fig. \ref{Fig1} (thick red line). 
The system is represented by the regression matrix $U \in \mathbb{R}^{200 \times 50}$ whose components are 
realizations of either white noise or  low
pass filtered white Gaussian noise with normalized band $[0,0.95]$. The measurement noise is white and Gaussian, with variance 
assumed known and set to 
the variance of the noiseless output divided by 10.

We use a Monte Carlo of 100 runs to compare the following two estimators
\begin{itemize}
\item {\bf Boosting}:  boosting estimator with $K$ set to the identity matrix and with 
$(\lambda,\nu)$ estimated using the SURE strategy (\ref{SURE}).
\item{ \bf Ridge:}  ridge regression (which corresponds to boosting with $\nu$ fixed to 1).
\end{itemize}

Fig. \ref{Fig2} displays the box plots of the 100 percentage fits of $\theta$,
$100\left(1-\frac{\| \theta- \hat{\theta} \|}{\| \theta \|}\right)$,
obtained by {\bf Boosting} and {\bf Ridge}.
When the entries of $U$ are white noise (left panel) 
one can see that the two estimators have similar performance.
When the entries of $U$ are filtered white noise (right panel) 
{\bf Boosting} performs significantly better than {\bf Ridge}. 
Furthermore, 36 out of the 100 fits achieved by {\bf Boosting}  
under the white noise scenario are lower
than those obtained adopting a low pass $U$, which is surprising since the conditioning 
of latter problem is much worse.
The reasons are those previously described. 
The unknown $\theta$ represents a smooth signal. In Bayesian terms, setting $K$ to the identity matrix 
corresponds to modeling it as white noise, which is a poor prior.
If the nature of $U$ is low pass, the energy of the $d_i^2$ are more concentrated at low frequencies.
So, as $\nu$ increases, {\bf Boosting} can inflate the prior variances associated to the low-frequency components
of $\theta$. The prior variances associated to high-frequencies induce low $SNR_i$, so that they increase 
slowly with $\nu$. This does not happen in the white noise case, since the random variables $d_i^2$ have 
similar distributions. Hence, the original white noise prior for $\theta$ can be significantly 
refined only in the low pass context: it is reshaped so as to form 
a regularizer, inducing more smoothness. Fig. \ref{Fig1} shows this effect by plotting  
estimates from {\bf Ridge} and {\bf Boosting} in a Monte Carlo run where $U$ is low pass.

\section{Boosting algorithms for general loss functions and RKHSs}\label{sec:new_class_boosting}

In this section, we combine the boosting kernel with piecewise linear-quadratic (PLQ) losses to 
obtain tractable algorithms for more general regression and classification problems. We also 
consider estimation in Reproducing Kernel Hilbert (RKHS) spaces. 

\subsection{Boosting kernel-based estimation with general loss functions}

\begin{figure}[h!]
\center
\begin{tabular}{ccccc}\\ 
\begin{tikzpicture}
  \begin{axis}[
    thick,
    height=2cm,
    xmin=-2,xmax=2,ymin=0,ymax=1,
    no markers,
    samples=50,
    axis lines*=left, 
    axis lines*=middle, 
    scale only axis,
    xtick={-1,1},
    xticklabels={},
    ytick={0},
    ] 
\addplot[blue, domain=-2:+2]{.5*x^2};
  \end{axis}
  \end{tikzpicture}
&\begin{tikzpicture}
  \begin{axis}[
    thick,
    height=2cm,
    xmin=-2,xmax=2,ymin=0,ymax=1,
    no markers,
    samples=100,
    axis lines*=left, 
    axis lines*=middle, 
    scale only axis,
    xtick={-1,1},
    xticklabels={},
    ytick={0},
    ] 
\addplot[blue-violet, domain=-2:+2]{sqrt(1+(x/1)^2) - 1};
  \end{axis}
\end{tikzpicture}
&\begin{tikzpicture}
  \begin{axis}[
    thick,
    height=2cm,
    xmin=-2,xmax=2,ymin=0,ymax=1,
    no markers,
    samples=50,
    axis lines*=left, 
    axis lines*=middle, 
    scale only axis,
    xtick={-1,1},
   xticklabels={},
    ytick={0},
    ] 
\addplot[red,domain=-2:0.5,densely dashed]{0*x};
\addplot[red,domain=0.5:+2,densely dashed]{x-.5};
  \end{axis}
\end{tikzpicture}\\
  (a) quadratic 
& (b)  huber 
& (d)  hinge \\
 \begin{tikzpicture}
  \begin{axis}[
    thick,
    height=2cm,
    xmin=-2,xmax=2,ymin=0,ymax=1,
    no markers,
    samples=100,
    axis lines*=left, 
    axis lines*=middle, 
    scale only axis,
    xtick={-.24,.56},
    xticklabels={},
    ytick={0},
    ] 
\addplot[red,domain=-2:-2*0.3*0.4,densely dashed]{0.3*abs(x) - 0.4*0.3^2};
\addplot[blue,domain=-2*0.3*0.4:2*(1-0.3)*0.4]{0.25*x^2/0.4};
\addplot[red,domain=2*(1-0.3)*0.4:2,densely dashed]{(1-0.3)*abs(x) - 0.4*(1-0.3)^2};
\addplot[blue,mark=*,only marks] coordinates {(-.24,0.0550) (0.56,0.20)};
  \end{axis}
\end{tikzpicture} 
& \begin{tikzpicture}
  \begin{axis}[
    thick,
    height=2cm,
    xmin=-2,xmax=2,ymin=0,ymax=1,
    no markers,
    samples=50,
    axis lines*=left, 
    axis lines*=middle, 
    scale only axis,
    xtick={-0.5,0.5},
    xticklabels={},
    ytick={0},
    ] 
    \addplot[red,domain=-2:-0.5,densely dashed] {-x-0.5};
    \addplot[domain=-0.5:+0.5] {0};
    \addplot[red,domain=+0.5:+2,densely dashed] {x-0.5};
    \addplot[blue,mark=*,only marks] coordinates {(-0.5,0) (0.5,0)};
  \end{axis}
\end{tikzpicture}
& \begin{tikzpicture}
  \begin{axis}[
    thick,
    height=2cm,
    xmin=-2,xmax=2,ymin=0,ymax=1,
    no markers,
    samples=100,
    axis lines*=left, 
    axis lines*=middle, 
    scale only axis,
    xtick={-1,1},
    xticklabels={},
    ytick={0},
    ] 
\addplot[amethyst, domain=-2:+2]{.5*x^2 + 0.5*abs(x)};
  \end{axis}
\end{tikzpicture}
\\ 
(e) quantile 
& (f) vapnik  
& (h) elastic net 
\end{tabular}
\caption{\label{fig:SDRex} Six common piecewise-linear quadratic losses. }
\end{figure}
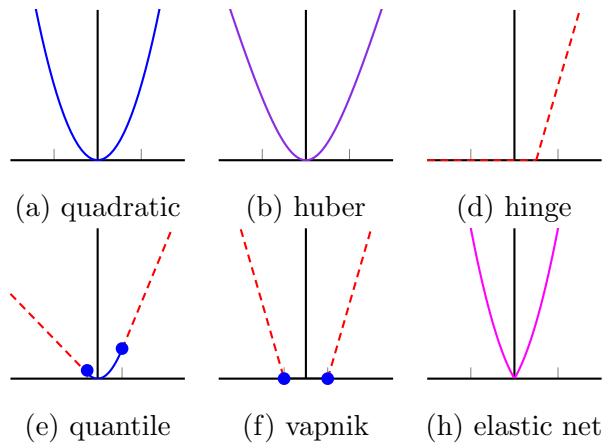

In the previous sections, the boosting kernel was derived using regularized least squares (\ref{Eq1}) as the
weak learner. The sequence of resulting linear estimators then led 
to a closed form expression for
$P_{\lambda,\nu}$. 
Now, we consider a kernel-based weak learner~\eqref{Weak2}, 
based on a general (convex) penalty $\mathcal{V}$. 
Important examples include  Vapnik's epsilon insensitive loss (Fig.~\ref{fig:SDRex}f) 
used in support vector regression~\citep{Vapnik98,Hastie01,Scholkopf00,Scholkopf01b},
hinge loss (Fig.~\ref{fig:SDRex}d) used for classification~\citep{Evgeniou99,Pontil98,Scholkopf00},
Huber and quantile huber (Fig~\ref{fig:SDRex}b,e), 
used for robust regression\citep{Hub,Mar,Bube2007,Zou08,KG01,Koenker:2005,aravkin2014qh}, 
and elastic net (Fig.~\ref{fig:SDRex}f),  a sparse regularizer 
that also finds correlated predictors~\citep{ZouHuiHastie:2005,EN_2005,li2010bayesian,de2009elastic}. 
The resulting boosting scheme is computationally expensive:
$\hat{y}(\nu=m)$ requires solving a sequence of 
$m$ optimization problems, each of which must be solved iteratively. 
In addition, since the estimators $\hat{y}(\nu)$
are no longer linear, deriving a boosting kernel is no longer straightforward. 

We combine general loss $\mathcal{V}$ with the regularizer induced by the 
boosting kernel from the linear case to define a new class of kernel-based boosting algorithms. 
More specifically, given a kernel $K$, let $VDV^T$ be the SVD of $UKU^T$.  If 
$P_{\lambda,\nu}$ is invertible,
the boosting output estimate is 
$\hat{y}(\nu) = U \hat{\theta}(\nu)$ where
\begin{eqnarray} \nonumber
\hat{\theta}(\nu) &=& \arg\min_{\theta} \ \mathcal{V}(y - U \theta) + \sigma^2 \theta^T U^T P_{\lambda,\nu}^{-1} U \theta \\ \label{Eq6}
&=&     \arg\min_{\theta}  \left\{  \mathcal{V}(y - U \theta) +  \theta^T U^T V \left[ \left(\frac{\lambda D + \sigma^2 I}{\sigma^2}\right)^{\nu}-I\right]^{-1}V^T U \theta\right\},
\end{eqnarray}
where the last line is obtained using (\ref{BoostingKer2}). Note, here and also in the reformulations below, 
that the solution depends on $\lambda$ and $\sigma^2$ only through the ratio $\gamma=\sigma^2/\lambda$.\\ 
If  $P_{\lambda,\nu}$ is not invertible, the following two strategies can be adopted.
\paragraph{Approach I:} We use 
\eqref{BoostingKer2} to obtain the factorization
$$
P_{\lambda,\nu} = \sigma^2 A_{\lambda,\nu} A_{\lambda,\nu}^T,
$$
where $A_{\lambda,\nu}$ is full column rank and
contains the columns of the matrix 
$$
A_{\lambda,\nu}  = V \left[ \left(\frac{\lambda D + \sigma^2 I}{\sigma^2}\right)^{\nu}-I\right]^{1/2}
$$
associated to the $d_i>0$. 
Then, the output estimate is $\hat{y}(\nu)= A_{\lambda,\nu} \hat{a}(\nu)$ with
\begin{equation}  \label{Eq7}
\hat{a}(\nu) = \arg\min_{a} \left\{ \mathcal{V}(y - A_{\lambda,\nu} a) + a^T a \right\}.
\end{equation}
The estimate of $\theta$ is then given by 
$\hat{\theta} =  U^{\dag}_{\lambda,\nu} \hat{y}(\nu)$,
where $U^{\dag}_{\lambda,\nu}$ is the pseudo-inverse of $U_{\lambda,\nu}$.
One advantage of the formulation (\ref{Eq7}) is that the evaluation of $A_{\lambda,\nu}$
for different $\lambda$ and $\nu$ is efficient.\\
\paragraph{Approach II:} Define the matrix 
$$
B_{\lambda,\nu}=U P_{\lambda,\nu} U^T.
$$ 
Then, it is easy to see that 
another representation for the output estimate is $\hat{y}(\nu)= B_{\lambda,\nu} \hat{b}(\nu)$ with
\begin{equation}  \label{Eq8}
\hat{b}(\nu) = \arg\min_{b} \left\{ \mathcal{V}(y - B_{\lambda,\nu} b) + b^T B_{\lambda,\nu}  b\right\}.
\end{equation} 
\bigskip

The new class of boosting kernel-based estimators defined by (\ref{Eq7}) or (\ref{Eq8})
keeps the advantages of boosting in the quadratic case.
In particular, the kernel structure can decrease 
bias along directions less exposed to noise.
The use of a general loss $\mathcal{V}$ allows a range of applications, 
with e.g. penalties such as Vapnik and Huber,
guarding against outliers in the training set. 
Finally, the algorithm has clear computational advantages
over the classic scheme described in Section \ref{SecBS}. 
Whereas in the classic approach, $\hat{y}(\nu=m)$ require solving 
$m$ optimization problems, in the new approach, given any positive  $\lambda$ and $\nu \geq 1$, 
the prediction $\hat{y}(\nu=m)$ is obtained by solving the single convex optimization problem (\ref{Eq6}).
This  is illustrated in Section \ref{sec:experiments}.  

\subsection{New boosting algorithms in RKHSs} \label{sec:RKHS}

We now show how the new class of boosting algorithms can be extended to the context
of regularization in RKHSs. We start with $\ell_2$ Boost in RKHSs.

Assume that we want to reconstruct a function 
from $n$ sparse and noisy data $y_i$ collected on input locations $x_i$ taking values
on the input space $\mathcal{X}$. 
Our aim now is to allow the function estimator to assume values in infinite-dimensional spaces, 
introducing suitable regularization to circumvent ill-posedness, e.g. in terms of function smoothness.
For this purpose, 
we use $\mathcal{K}$ denote a kernel function 
$\mathcal{K}: \mathcal{X} \times \mathcal{X} \rightarrow \mathbb{R}$
which captures smoothness properties of the unknown function. 
We can then use $\ell_2$ Boost, with weak learner
\begin{equation}\label{WeakRKHS}
\argmin_{f \in \mathcal{H}} \ \sum_{i=1}^n \ \mathcal{V}_i(y_i-f(x_i)) + \gamma \| f \|^2_{\mathcal{H}},
\end{equation} 
where $\mathcal{V}_i$ is a generic convex loss and 
$\mathcal{H}$ is the RKHS induced by $\mathcal{K}$ with norm denoted by 
$\| \cdot \|_{\mathcal{H}}$. From the representer theorem  of \cite{Scholkopf01}, the solution of 
(\ref{WeakRKHS}) is $\sum_{i=1}^n \ \hat{c}_i  \mathcal{K}(x_i,\cdot)$
where the $\hat{c}_i$ are the components of the column vector
\begin{equation}\label{WeakRKHS2}
\argmin_{c \in \mathbb{R}^n} \ \sum_{i=1}^n \ \mathcal{V}_i(y_i- K_{i,\cdot} c ) + \gamma c^T K c,
\end{equation} 
and $K$ is the kernel (Gram) matrix, with $K_{i,j} = \mathcal{K}(x_i,x_j)$ and 
$K_{i,\cdot}$ is the $i$-th row of $K$.  
Using (\ref{WeakRKHS2}), we extend the boosting scheme from 
section \ref{SecBS} with (\ref{WeakRKHS}) as  the weak learner.
In particular, repeated applications of the representer theorem ensure that, for any 
value of the iteration counter $\nu$,
the corresponding function estimate belongs to the subspace 
spanned by the $n$ kernel sections $\mathcal{K}(x_i,\cdot)$. 
Hence, $\ell_2$ Boosting in RKHS can be summarized as follows. 

{\bf{Boosting scheme in RKHS:}}
\begin{enumerate}
\item Set $\nu=1$. Solve (\ref{WeakRKHS2}) to obtain $\hat{c}$ and $\hat{f}$ for $\nu=1$, call them
$\hat{c}(1) $ and $\hat{f}(\cdot,1)$.
\item Update $c$ by solving (\ref{WeakRKHS2}) with the current residuals as the data vector:
$$
\hat{c}(\nu+1) = \hat{c}(\nu) + \argmin_{c \in \mathbb{R}^n} \ \sum_{i=1}^n \ \mathcal{V}_i(y_i- K_i \hat{c}(\nu) - K_i c ) + \gamma c^T K c,
$$
and set the new estimated function to
$$
\hat{f}(\cdot, \nu+1) =  \sum_{i=1}^n \ \hat{c}_i(\nu+1)  \mathcal{K}(x_i,\cdot). 
$$
\item Increase $\nu$ by 1 and repeat step 2 for a prescribed number of iterations.
\end{enumerate}

There is a fundamental computational drawback related to this scheme 
which we have already encountered in the previous sections. 
To obtain $\hat{f}(\cdot, \nu)$ we need to solve $\nu$ optimization problems, 
each of them requiring an iterative procedure.
Now, we define a new computationally efficient class of 
regularized estimators in RKHS.
The idea is to 
obtain the expansion coefficients of the function estimate through the new boosting kernel. 
Letting $\gamma=\sigma^2 / \lambda$ and $P_{\lambda}= \lambda K$, with $K$ 
the kernel matrix, define the boosting kernel $P_{\lambda,\nu}$
as in (\ref{BoostingKer}). Then, we can first solve 
\begin{equation}  \label{Eq9}
\hat{b}(\nu) = \arg\min_{b} \left\{ \mathcal{V}(y - P_{\lambda,\nu} b) + b^T P_{\lambda,\nu}  b\right\},
\end{equation} 
with $\mathcal{V}$ defined as the sum of the $\mathcal{V}_i$.  
Then, we compute 
$$
\tilde{c}=K^{\dag}\tilde{y}(\nu) \quad \mbox{with} \quad  \tilde{y}(\nu) = P_{\lambda,\nu} \hat{b}(\nu), 
$$
and the estimated function becomes 
$$
\hat{f}(\cdot, \nu) =  \sum_{i=1}^n \ \tilde{c}_i(\nu)  \mathcal{K}(x_i,\cdot). 
$$
Note that the weights $\tilde{c}(\nu)$ coincide with $\hat{c}(\nu)$ only when the
$\mathcal{V}_i$ are quadratic. Nevertheless, given any loss, (\ref{Eq9}) preserves
all advantages of boosting outlined in the linear case.
Furthermore, as in the finite-dimensional case, 
given any $\nu$ and kernel hyperparameter, the estimator (\ref{Eq9}) can compute $\tilde{c}(\nu)$  
by solving a single problem, rather than iterating the boosting scheme. 

\paragraph{Classification with the hinge loss.}
Another advantage related to the use of the boosting kernel w.r.t. 
the classical boosting scheme arises in the classification context.
Classification tries to predict one of two output values, e.g. 1 and -1, 
as a function of the input. $\ell_2$ Boost could be used 
using the residual $y_i-f(x_i)$ as misfit, e.g. 
equipping the weak learner (\ref{WeakRKHS}) with the quadratic
or the $\ell_1$ loss.  
However, in this context one often prefers to use the margin $m_i=y_if(x_i)$
on an example $(x_i,y_i)$ to measure how well 
the available data are classified. For this purpose, 
support vector classification is widely used \citep{scholkopf2002learning}. It relies on the
hinge loss 
\[
\mathcal{V}_i(y_i,f(x_i)) = | 1 - y_i f(x_i) |_+ =
\left\{
\begin{array}{lcl}
0, \quad & m > 1 \\
1-m, \quad & m \leq 1
\end{array}, \quad m=y_if(x_i),
\right.
\]
which gives a linear penalty when $m<1$. Note that this loss
assumes $y_i\in\{1,-1\}$. 
However, the classical  boosting scheme applies  the weak learner (\ref{WeakRKHS}) repeatedly, and 
 {\bf residuals will not be binary} for $\nu>1$. 
 This means that $\ell_2$ Boost cannot be used for the hinge loss.
 
 This limitation does not affect the new class of boosting-kernel based estimators:
support vector classification can be boosted
by plugging in the hinge loss into (\ref{Eq9}): 
\begin{equation}  \label{Eq10}
\hat{b}(\nu) = \arg\min_{b} \ \sum_{i=1}^n  | 1 - y_i [P_{\lambda,\nu} b]_i |_+ + b^T P_{\lambda,\nu}  b,
\end{equation} 
where we have used $[P_{\lambda,\nu} b]_i $ to denote the $i$-th component of $P_{\lambda,\nu} b$.

\section{Numerical Experiments} \label{sec:experiments}

\subsection{Boosting kernel regression: temperature prediction real data}\label{sec:RealData}

\begin{figure}[h!]
  \begin{center}
   \begin{tabular}{cc}
\hspace{-.2in}
 { \includegraphics[scale=0.35]{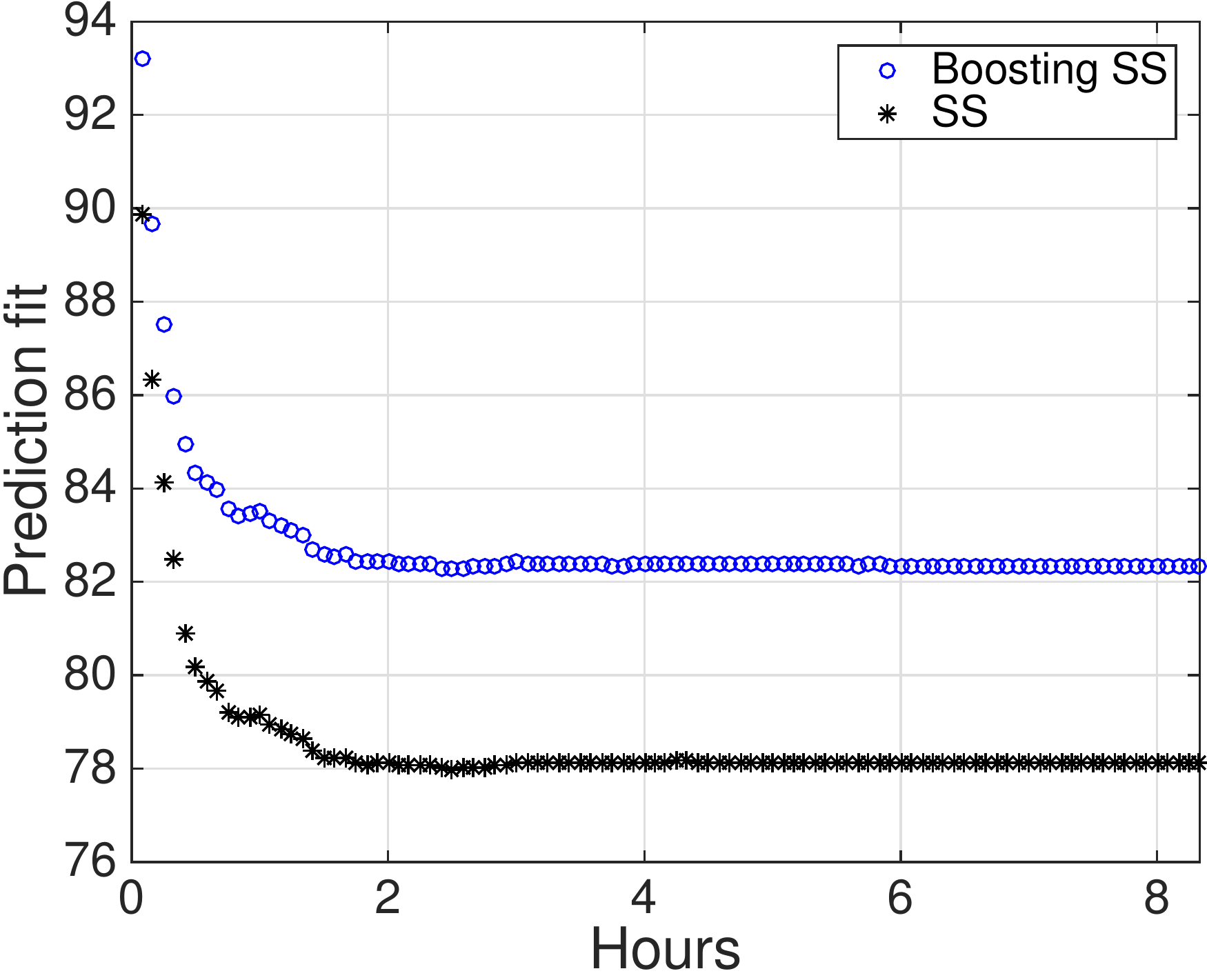}}
\hspace{.1in}
 { \includegraphics[scale=0.35]{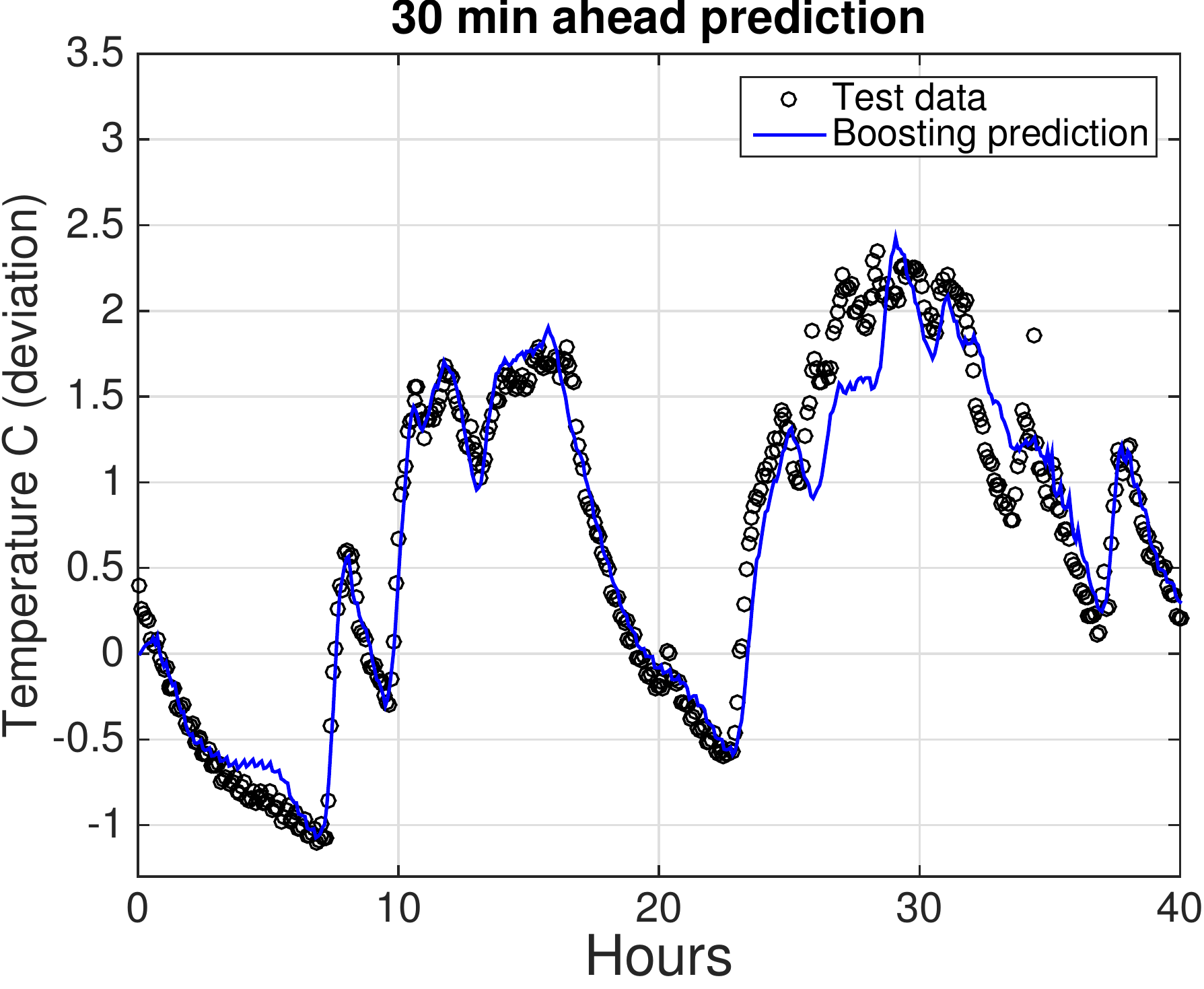}}
    \end{tabular}
    \caption{\emph{Left:} prediction fits obtained by the stable spine estimator (SS) and by Boosting 
    equipped with the stable spline kernel (Boosting SS). \emph{Right:} 30-min ahead temperature prediction from Boosting SS
    on a portion of the test set.}  \label{Fig3}
     \end{center}
\end{figure}

To test boosting on real data, we use a case study in thermodynamic modeling of 
buildings. Eight temperature sensors produced by Moteiv Inc were placed in two rooms of a small two-floor residential building of about $80$ $\textrm{m}^2$ and $200$ $\textrm{m}^3$. 
 The experiment  lasted for 8 days starting from February 24th, 2011; samples
were taken every 5 minutes. A thermostat controlled the heating systems and the reference temperature
was manually set every day depending upon occupancy and other needs. The goal of the experiment is to assess the predictive capability of models built using kernel-based estimators.

We consider Multiple Input-Single Output (MISO) models.
The temperature from the first node is the output ($y_i$) and the other
7 represent the inputs ($u^j_i$, $j=1,..,7$).
The measurements are split into a training set of size $N_{id}=1000$ and a test set of size $N_{test}=1500$. 
The notation $y^{test}$ indicates the test data, which is used to test the ability of our estimator to predict future data.
Data are normalized so that they have zero mean and unit variance before identification is performed.

The model predictive power is measured in terms of $k$-step-ahead prediction fit on $y^{test}$, i.e.
$$
100 \times \left(1-{\sqrt{\sum_{i=k}^{N_{test}}(y^{test}_i-\hat y_{i|i-k})^2}/ \sqrt{\sum_{i=k}^{N_{test}} (y^{test}_i)^2}}\right).
$$

We consider ARX models of the form 
 $$
 y_i = (g^1 \otimes y)_i  + \sum_{j=1}^{7} (g^{j+1} \otimes u^j)_i + v_i,
$$
where $\otimes$ denotes discrete-time convolution and the $\{g^j\}$ are  
8 unknown one-step ahead predictor impulse responses, each of length 50. 
Note that when such impulse responses are known, one can use them in an iterative fashion
to obtain any $k$-step ahead prediction.
We can stack all the $\{g^j\}$ in the vector $\theta$ and form the regression matrix $U$ with the past outputs and the inputs
so that the model becomes $y=U\theta + v$.
Then, we consider the following two estimators:
\begin{itemize}
\item {\bf Boosting SS}: this estimator regularizes
each $g^j$ introducing information on its smoothness 
and exponential decay by the stable spline kernel \citep{SS2010}. 
In particular, let $P \in \mathbb{R}^{50 \times 50}$ with $(i,j)$ entry $\alpha^{\max(i,j)}, \ 0 \leq \alpha <1$.  
Then, we recover $\theta$ by the boosting scheme (\ref{Eq7}) with $K=\mbox{blkdiag}(P,\ldots,P)$,
and $\mathcal{V}$ set to the quadratic loss. 
Note that the estimator contains the three unknown hyperparameters
$\nu,\alpha$ and $\gamma=\sigma^2/\lambda$. 
To estimate them, the training set is divided in half and hold-out cross validation is used. 
\item {\bf Classical Boosting SS}: the same as above except that $\nu$ can assume only integer values.
 \item {\bf SS}: this is the stable spline estimator described in \citep{SS2010}
(and corresponds to {\bf Boosting SS} with $\nu=1$) with hyperparameters obtained via marginal
likelihood optimization.
\end{itemize}

For {\bf Boosting SS}, we obtained $\gamma=0.02, \alpha=0.82$ and $\nu = 1.42$; 
note that it is not an integer. For {\bf Classical Boosting SS}, we obtained 
$\gamma=0.03, \alpha=0.79$ and $\nu = 1$. In practice, this
estimator gives the same results achieved by {\bf SS}
so that our discussion below just compares the performance of
{\bf Boosting SS} and {\bf SS}.

The left panel of Fig. \ref{Fig3} shows the prediction fits,
as a function of the prediction horizon $k$, obtained by {\bf Boosting SS} and {\bf SS}. 
Note that the non-integer $\nu$ gives an improvement 
in performance. 
This means that in this experiment 
using a continuous $\nu$  improves also over the classical boosting. 
The right panel of Fig. \ref{Fig3}  shows sample trajectories of half-hour-ahead boosting prediction on a part of the test set.

\subsection{Boosting kernel regression using the $\ell_1$ loss: Real data water tank system identification} \label{sec:real_experiment}

\begin{figure}[h!]
  \begin{center}
   \begin{tabular}{cc}
\hspace{-.2in}
 { \includegraphics[scale=0.35]{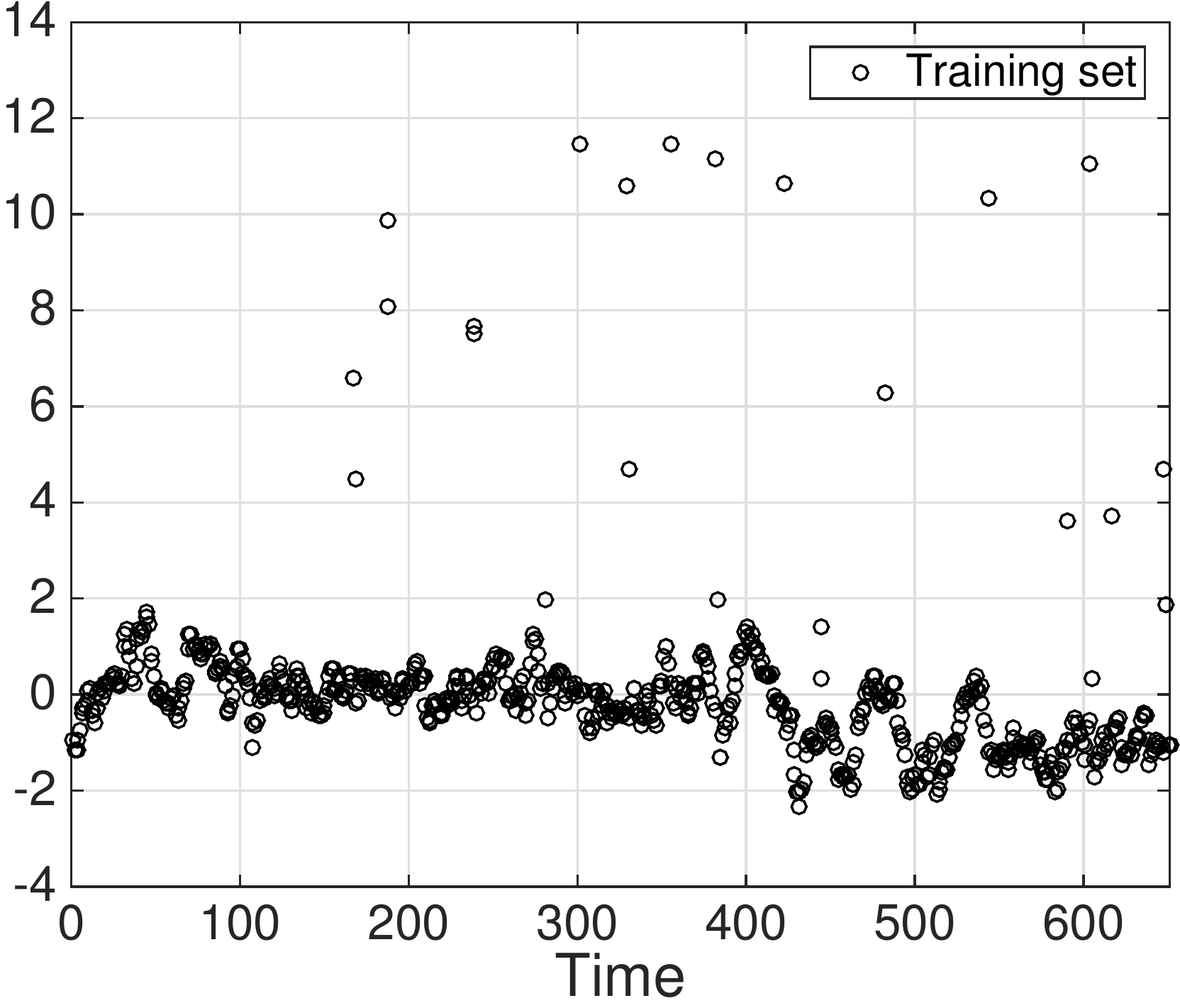}}
\hspace{.1in}
 { \includegraphics[scale=0.35]{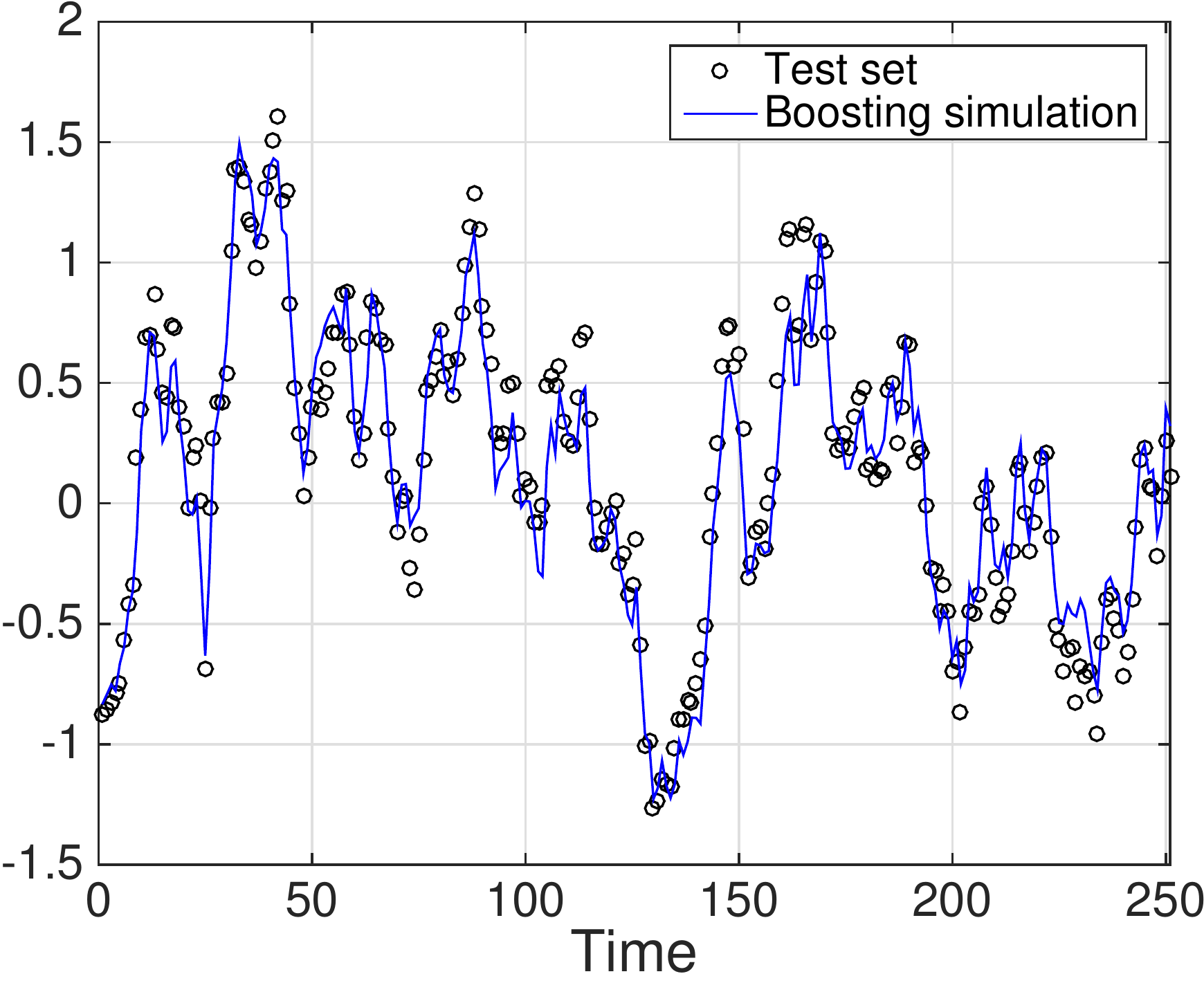}}
    \end{tabular}
   \caption{\emph{Left:} training set. \emph{Right:} test set simulation from Boosting SS  with $\ell_1$ loss.}  \label{Fig4}
     \end{center}
\end{figure}
We test our new class of boosting algorithms on another real data set
obtained from a water tank system (see also \cite{bottegal2016robust}). 
In this example, a tank is fed with water by an electric pump. 
The water is drawn from a lower basin, and then flows back through a hole in the bottom of the tank. 
The system input is the voltage applied, while
the output is the water level in the tank, 
measured by a pressure sensor at the bottom of the tank. 
{The setup represents a typical control engineering scenario, where the experimenter is interested in building a mathematical model of the system in order to predict its behavior and design a control algorithm \citep{Ljung}. To this end,} 
input/output samples are collected every second, comprising almost 1000 pairs
that are divided into a training and test set. The signals are de-trended, removing their means.
The training and test outputs are shown in the left and right panel
of Fig. \ref{Fig4}. One can see that the
second part of the training data are corrupted by outliers caused by pressure perturbations in the tank;
these are due to air occasionally being blown into the tank. {Our aim is to understand the predictive capability of the boosting kernel even in presence of outliers.}

We consider a FIR model of the form 
 $$
 y_i = (g \otimes u)_i  + v_i,
$$
where the unknown vector  $g \in \mathbb{R}^{50}$ contains the impulse response
coefficients. It is estimated using a variation of the estimator {\bf Boosting SS} described in the previous section:
while the stable spline kernel is still employed  to define the regularizer, the key difference is that  $\mathcal{V}$
in (\ref{Eq7}) is now set to the robust $\ell_1$ loss.
The hyperparameter estimates obtained using hold-out cross validation are $\gamma=17.18,\alpha=0.92$
and  $\nu=1.7$. 
The right panel of Fig. \ref{Fig4}  shows the boosting simulation of the test set.
The estimate from {\bf Boosting SS} predicts the test set with $76.2 \%$ fit.
Using the approach $\mathcal{V}$ equal to the quadratic loss,
the test set fit decreases to $57.8 \%$. 

\subsection{Boosting in RKHSs: Classification problem} \label{sec:class_rkhs}

Consider the problem described in Section 2 of \citep{hastie2001elements}.
Two classes are introduced, each defined by a mixture of Gaussian clusters;
the first 10 means are generated from a Gaussian $\mathcal{N}([1 \ 0]^T, I)$
and remaining ten means from $\mathcal{N}([0 \ 1]^T, I)$ with $I$ the identity matrix.
Class labels $1$ and $-1$ corresponding to the clusters are generated randomly with probability 1/2. 
Observations for a given label are generated by picking one of the ten means $m_k$ from the correct cluster with uniform probability 1/10,
and drawing an input location from $\mathcal{N}(m_k, I/5)$.
A Monte Carlo study of 100 runs is designed.
At any run, a new data set of size 500 is generated, 
with the split given by $50\%$ for training   
and $25\%$ each for validation and testing.
The validation set is used to estimate through
hold-out cross-validation the unknown hyperparameters,
in particular the boosting parameter $\nu$. 
Performance for a given run is quantified by 
computing percentage of data correctly classified.

We compare the performance of the following two estimators:
\begin{itemize}
\item {\bf Boosting+$\ell_1$ loss}: this is the boosting scheme in RKHS illustrated in
the previous section ($\nu$ may assume only integer values) with the weak learner (\ref{WeakRKHS}) 
defined by the Gaussian kernel
$$
K(x,a) = \exp(-10 |x-a|^2), \quad | \cdot |=\mbox{Euclidean norm}
$$
setting each $\mathcal{V}_i$ to the $\ell_1$ loss and using $\gamma=1000$.
\item {\bf Boosting kernel+$\ell_1$ loss}: this is the estimator
using the new boosting kernel. The latter 
 is defined by the kernel matrix built using the same Gaussian kernel
 reported above, with $\sigma^2=1,\lambda=0.001$ so that one still has $\gamma=1000$.
 The function estimate is achieved solving (\ref{Eq9}) using the $\ell_1$ loss.
\end{itemize}

Note that the two estimators contain only one unknown parameter, i.e. $\nu$ which is
estimated by the cross validation strategy described above.
The top left panel of Fig. \ref{FigTib1} compares their performance.
Interestingly, results are very similar, see also Table \ref{Table1}.
This supports the fact that the boosting kernel
can include classical boosting features in the estimation process. 
In this example, the difference between the two methods is mainly in their computational complexity.
In particular, the top right panel of Fig. \ref{FigTib1} reports  
some cross validation scores as a function of the boosting iterations counter $\nu$
for the classical boosting scheme. The score is linearly interpolated, since $\nu$ can assume only integer values.
On average, during the 100 Monte Carlo runs the optimal value 
corresponds to $\nu=340$, so on average, problems~(\ref{WeakRKHS}) 
must be solved 340 times. After obtaining the estimate of $\nu$, to obtain the function estimate
using the union of the training and validation data, another 340 problems must be solved.

In contrast, the boosting kernel used in (\ref{Eq9}) does not require repeated optimization of the weak learner. 
Using a golden section search,estimating  $\nu$  
by cross validation  on average requires solving 20 problems of the form (\ref{Eq9}).
Once $\nu$ is found, only one additional optimization problem must be solved
to obtain the function estimate. Summarizing, in this example the boosting kernel
obtains results similar to those achieved by classical boosting,
but requires solving only 20 optimization problems rather than nearly 700.
The computational times of the two approaches are reported in the bottom panel of Fig. \ref{FigTib1}.

Table \ref{Table1} also shows the average fit obtained by other two estimators. 
The first estimator is denoted by {\bf Boosting SVC}: it
coincides with {\bf Boosting kernel+$\ell_1$ loss},
except that the hinge loss replaces the $\ell_1$ loss in (\ref{Eq9}). 
The other one is {\bf SVC} and corresponds to the classical support
vector classifier. It uses the same Gaussian kernel defined above with the regularization parameter
$\gamma$ determined via cross validation 
on a grid containing 20 logarithmically spaced values on the interval $[0.01,100]$.  
One can see that the best results are obtained by
boosting support vector classification. Recall also that the hinge loss cannot be adopted
using the classical boosting scheme as discussed at the end of the previous section.

\begin{table*}
\begin{center}
\begin{tabular}{cccc}\hline
{\bf Boosting+$\ell_1$} & {\bf Boosting kernel+$\ell_1$}  & {\bf Boosting SVC} & {\bf SVC} \\
78.91 \%          & 79.15 \%  & 79.73 \%  & 78.12 \%\\
 \hline \phantom{|}
\end{tabular} 
\end{center}
\caption{Average percentage classification fit} \label{Table1}
\end{table*}

\begin{figure}[h!]
  \begin{center}
   \begin{tabular}{cc}
\hspace{-.2in}
 { \includegraphics[scale=0.35]{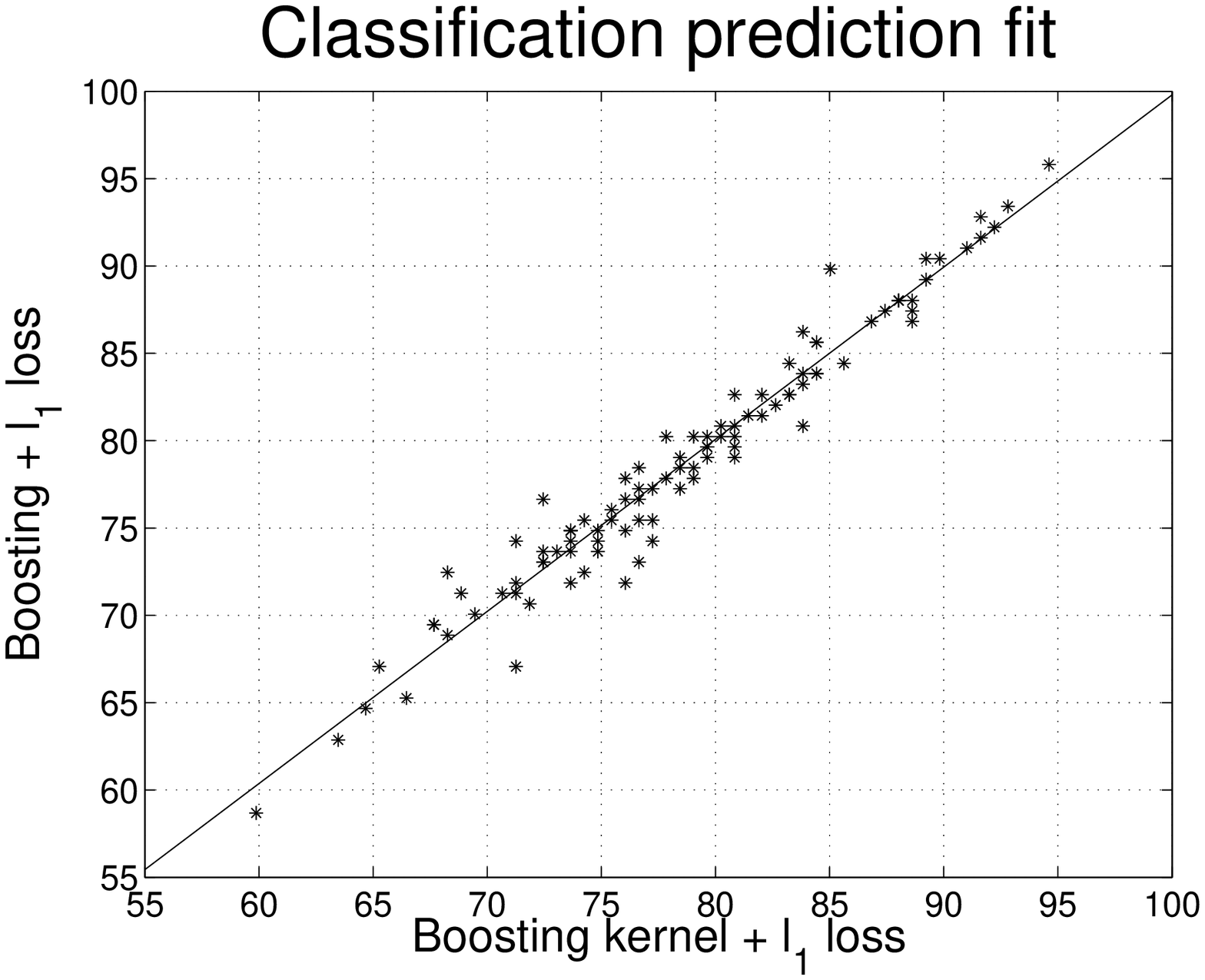}}
 \hspace{.1in}
 { \includegraphics[scale=0.35]{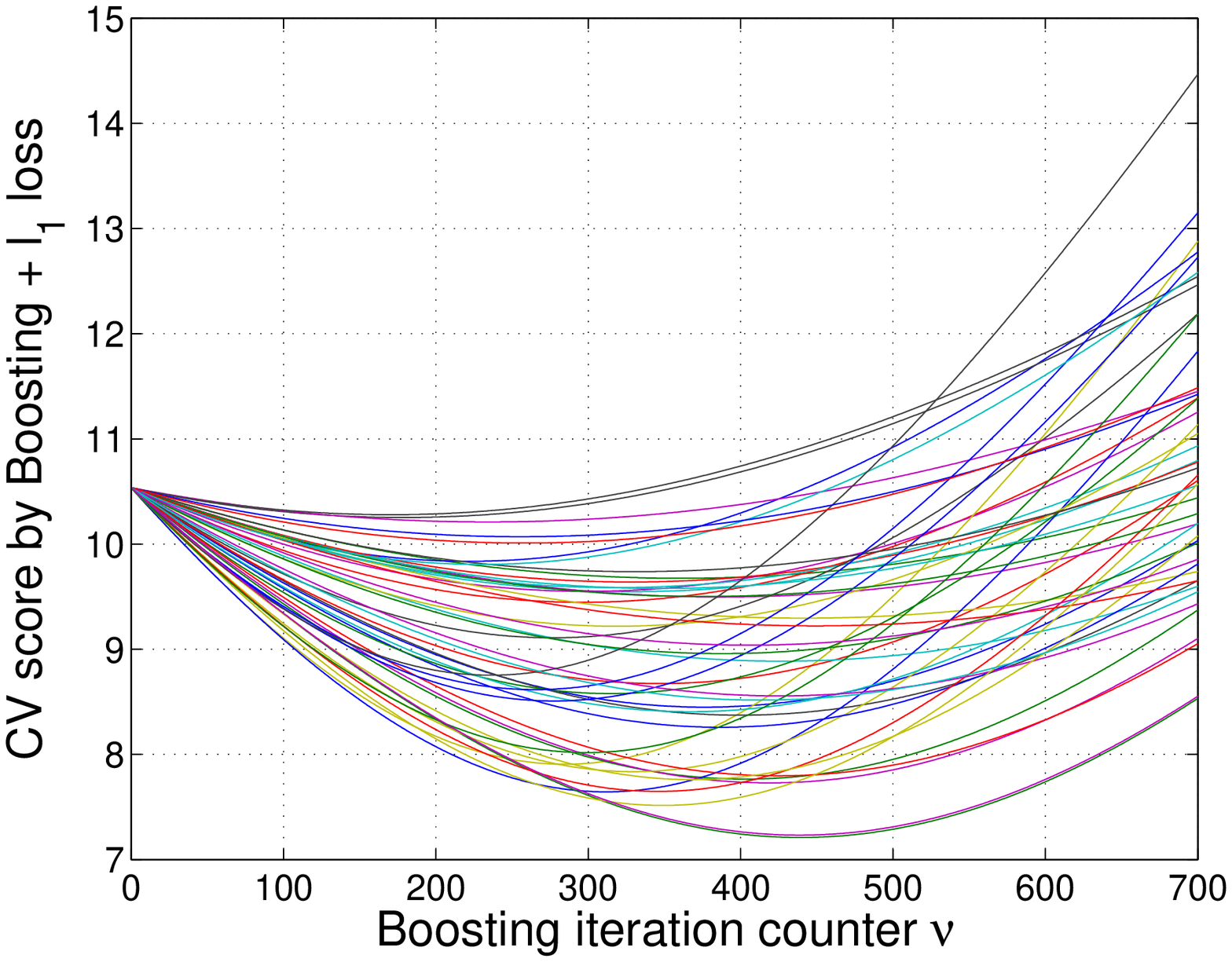}}\\
  \hspace{.1in}
 { \includegraphics[scale=0.35]{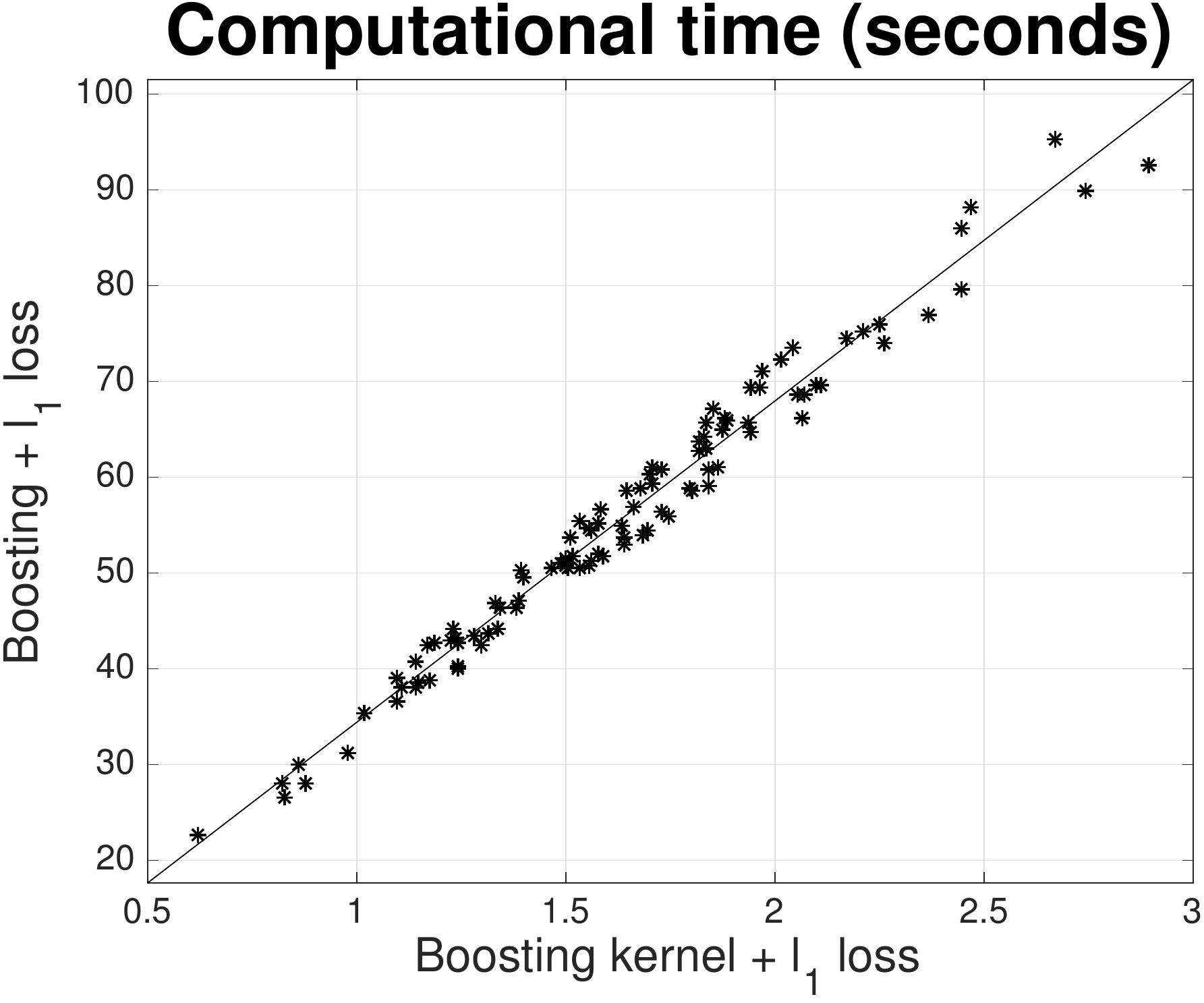}}
    \end{tabular}
   \caption{{\bf Classification problem} \emph{Top Left} Fits obtained by the new boosting kernel  (x-axis) vs 
   fits obtained by the classical boosting scheme (y-axis). Both the estimators use the $\ell_1$ loss.
   \emph{Top Right} Some cross validation scores computed using 
   the classical boosting scheme equipped with the $\ell_1$ loss as a function of the boosting iteration counter $\nu$. 
   Each curve corresponds to a different run. \emph{Bottom} Computational times 
   to solve a classification problem needed by the new boosting kernel  (x-axis) and
    by the classical boosting scheme (y-axis). }  \label{FigTib1}
     \end{center}
\end{figure}

\subsection{Boosting in RKHSs: Regression problem} \label{sec:regression_rkhs}

Consider now a regression problem where only smoothness information
is available to reconstruct the unknown function from sparse and noisy data. 
As in the previous example, our aim is to illustrate
how the new class of proposed boosting algorithms
can solve this problem using a RKHS with a great computational
advantage w.r.t. the traditional scheme.
For this purpose, we just consider a classical benchmark problem where 
the unknown map is the Franke's bivariate test function $f$ given by the
weighted sum of four exponentials \citep{Wahba1990}. 
Data set size is 1000 and is generated as follows.
First, 1000 input locations $x_i$ are drawn from a
uniform distribution on $[0,1] \times [0,1]$. 
The data are divided in the same way described in the classification problem.
The outputs in the training and validation data are 
$$
y_i = f(x_i) + v_i
$$
where the errors $v_i$ are independent, with distribution given by the mixture of Gaussians
$$
0.9 \mathcal{N}(0,0.1^2) + 0.1 \mathcal{N}(0,1). 
$$
The test outputs $y^{test}_i$ are instead given by noiseless outputs $f(x_i^{test})$.
A Monte Carlo study of 100 runs is considered, where a new data set is generated at any run.
The test fit is computed as
$$
100 \left( 1-\frac{| y^{test} -\hat{y}^{test}|}{|y^{test}- \mbox{mean}(y^{test}) |}\right),
$$
where $\hat{y}^{test}$ is the test set prediction.

Note that the mixture noise can model the effect of outliers which
affect, on average, 1 out of 10 outputs. This motivates the use of the robust $\ell_1$ loss. 
Hence, the function is still reconstructed by {\bf Boosting+$\ell_1$ loss}
and {\bf Boosting kernel+$\ell_1$ loss} which are implemented exactly in the same way as previously described.
Fig.  \ref{FigWahba1} displays the results with the same rationale adopted in
Fig.  \ref{FigTib1}. 
The fits are close each other but, at any run, the classical boosting scheme requires 
solving hundreds of optimization problems, while the boosting kernel-based approach needs to solve
around 15 problems on average. The computational times of the two approaches are reported in the bottom panel of Fig. \ref{FigWahba1}.

 Finally, Table \ref{Table2} reports the average fits including those achieved 
 by {\bf Gaussian kernel+$\ell_1$ loss}, which is implemented 
 as the estimator {\bf SVC} described in the previous section except that the
 hinge loss is replaced by the $\ell_1$ loss. The best results are achieved by boosting kernel with $\ell_1$.

\begin{table*}
\begin{center}
\begin{tabular}{cccc}\hline
{\bf Boosting+$\ell_1$} & {\bf Boosting kernel+$\ell_1$}  & {\bf Gaussian kernel+$\ell_1$}  \\
76.62 \%          & 76.75 \%  & 75.19 \%\\
 \hline \phantom{|}
\end{tabular} 
\end{center}
\caption{Average percentage regression fit} \label{Table2}
\end{table*}

\begin{figure}[h!]
  \begin{center}
   \begin{tabular}{cc}
\hspace{-.2in}
 { \includegraphics[scale=0.35]{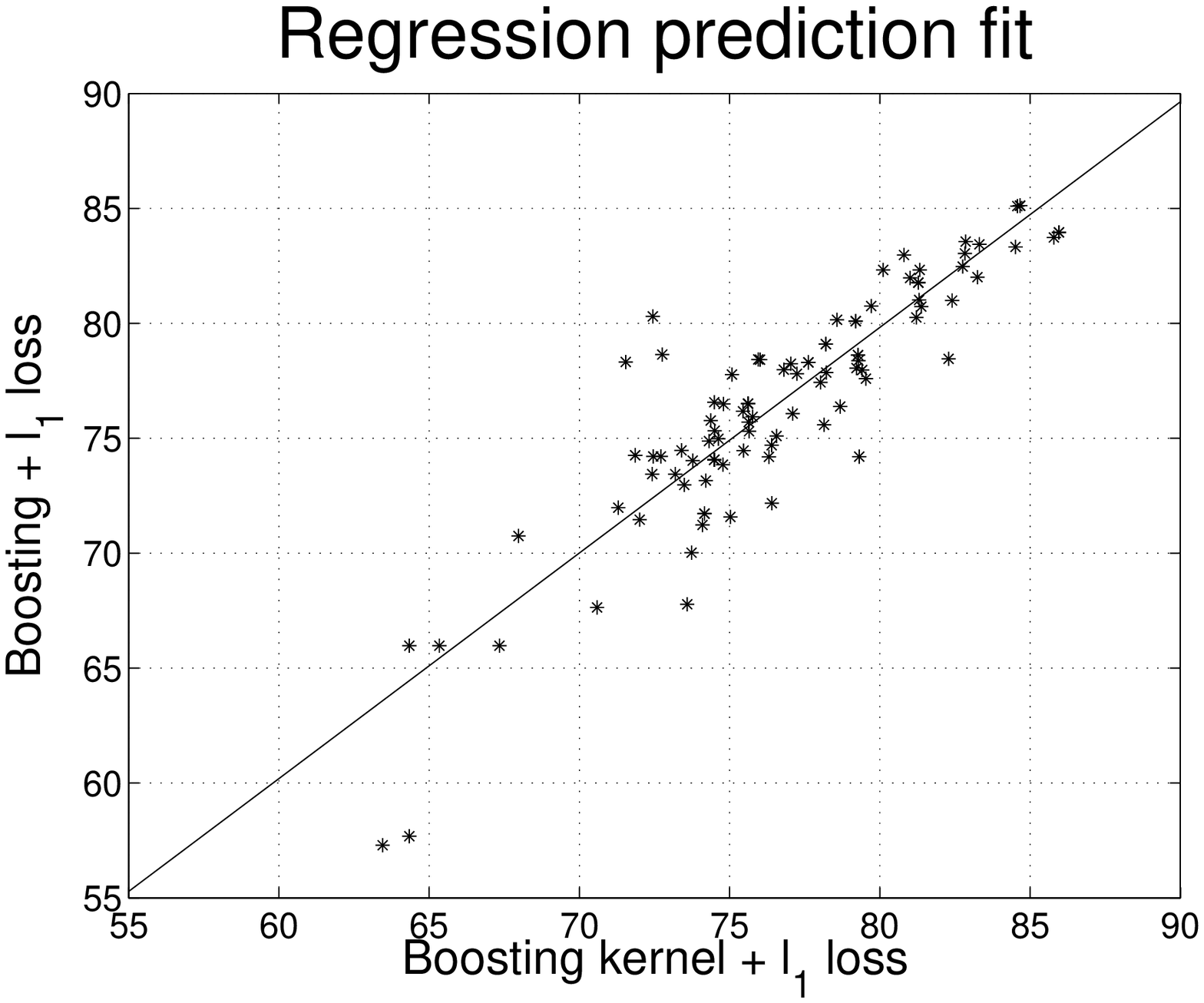}}
\hspace{.1in}
 { \includegraphics[scale=0.35]{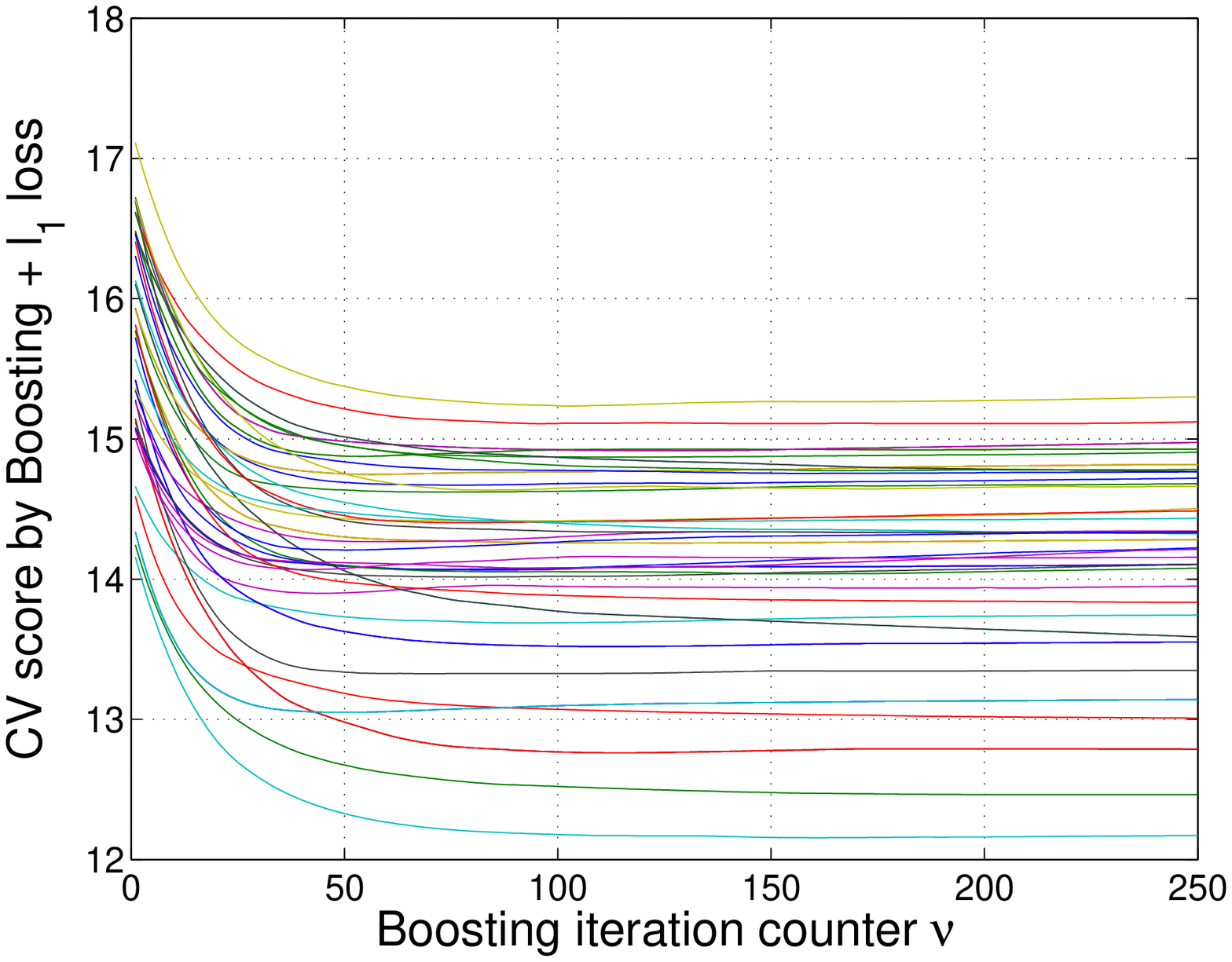}}\\
  \hspace{.1in}
 { \includegraphics[scale=0.35]{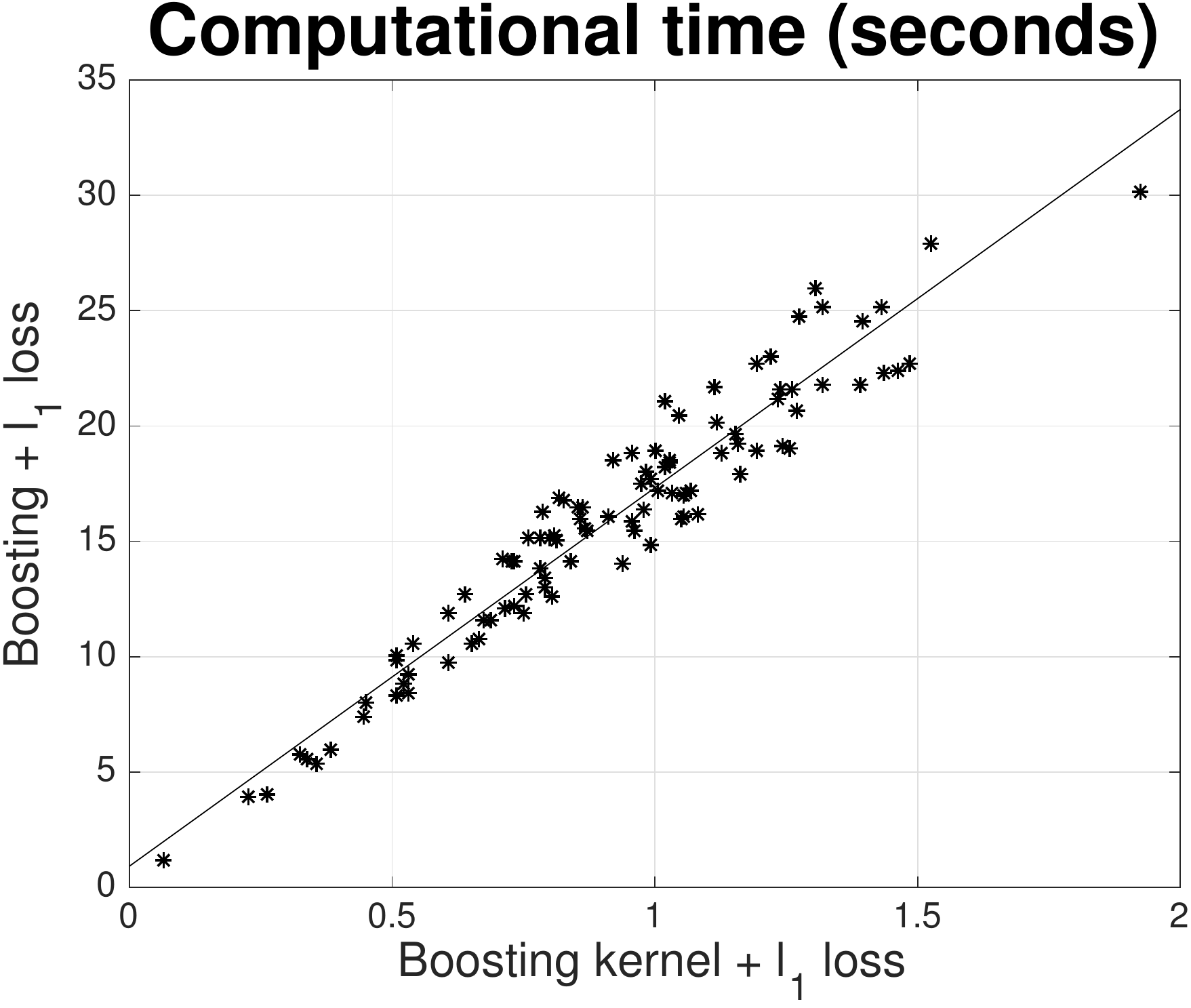}}
    \end{tabular}
    \caption{{\bf Regression problem} \emph{Top Left} Fits obtained by the new boosting kernel  (x-axis) vs 
   fits obtained by the classical boosting scheme (y-axis). Both the estimators use the $\ell_1$ loss.
   \emph{Top Right} Some cross validation scores computed using 
   the classical boosting scheme equipped with the $\ell_1$ loss as a function of the boosting iteration counter $\nu$.
 \emph{Bottom} Computational times 
   to solve a regression problem needed by the new boosting kernel  (x-axis) and
    by the classical boosting scheme (y-axis). }  
   \label{FigWahba1}
     \end{center}
\end{figure}

\section{Conclusion}
\label{sec:conclusions}

In this paper, we presented a connection between boosting and kernel-based methods.
We showed that in the context of regularized least-squares, boosting with a weak learner can be
interpreted using a boosting kernel. This connection was used for three main applications: 
(1) providing insight into boosting estimators and when they can be effective;
(2) determining schemes for hyperparameter estimation using the kernel connection and 
(3) proposing a more general class of boosting schemes for general misfit 
measures, including $\ell_1$, Huber and Vapnik, which can use also RKHSs
as hypothesis spaces. 

The proposed approach combines generality with 
computational efficiency. In contract to the classic boosting scheme, treating 
boosting iterations $\nu$ as a continuous hyperparameter 
may improve prediction capability. 
Real data support the use of these generalized schemes in practice.
Indeed, in some real experiments we obtained $\nu = 1.42$ as estimate improving  
on the classic scheme. 
In addition, this new viewpoint avoids sequential solutions.
This turns out a particularly strong advantage for boosting
using general losses $\mathcal{V}$, as each boosting run would itself require an iterative algorithm. 
This has been outlined also in the RKHS setting: the boosting kernel allows to obtain
results similar (or also better) than the classical boosting scheme dramatically reducing the computational cost.

\vskip 0.2in
\bibliography{boosting_biblio,references_sasha}

\end{document}

%% file: macro.tex
\newcommand{\argmin}{\operatornamewithlimits{argmin}}

\newlength{\widebarargwidth}
\newlength{\widebarargheight}
\newlength{\widebarargdepth}
